\documentclass[11pt]{article}

\usepackage[final]{acl}

\usepackage{times}
\usepackage{latexsym}

\usepackage[T1]{fontenc}

\usepackage[utf8]{inputenc}

\usepackage{microtype}

\usepackage{inconsolata}

\usepackage{graphicx}
\usepackage{subcaption}
\usepackage{amsmath}
\usepackage{algorithm}
\usepackage{algorithmic}
\usepackage{amssymb}
\usepackage{mathtools}
\usepackage{amsthm}
\usepackage{booktabs}
\usepackage[table]{xcolor}

\theoremstyle{plain}
\newtheorem{theorem}{Theorem}[section]

\theoremstyle{definition}

\theoremstyle{remark}

\title{Implicit Hierarchical GRPO: Decoupling Tool Invocation from Execution for Tool-Integrated Mathematical Reasoning}

\author{
 \textbf{Li Wang\textsuperscript{1}\thanks{Equal contribution}},
 \textbf{Xiaohan Wang\textsuperscript{1}\footnotemark[1]\thanks{Corresponding authors: Xiaohan Wang: \href{mailto:wangxiaohan17@meituan.com}{wangxiaohan17@meituan.com}, Guojun Yin: \href{mailto:yinguojun02@meituan.com}{yinguojun02@meituan.com}}},
 \textbf{Xiaodong Lu\textsuperscript{1}},
 \textbf{Zipeng Zhang\textsuperscript{2}},
 \textbf{Jinyang Wu\textsuperscript{2}},
\\
 \textbf{Jiajun Chai\textsuperscript{1}},
 \textbf{Wei Lin\textsuperscript{1}},
 \textbf{Guojun Yin\textsuperscript{1}\footnotemark[2]}
\\
 \textsuperscript{1}Meituan, \textsuperscript{2}Tsinghua University
}
\begin{document}
\maketitle

\begin{abstract}
Large language models (LLMs) have increasingly leveraged tool invocation to enhance their reasoning capabilities. However, existing approaches typically tightly couple tool invocation with immediate execution. Such immediate tool interaction may disrupt the reasoning coherence of LLMs and constrain their expressivity, ultimately degrading reasoning performance. To this end, for the first time, we propose and formalize the problem of decoupling tool invocation from execution during reasoning, and introduce delayed execution with explicit control to enhance tool-integrated reasoning (TIR). Furthermore, we propose a hierarchical control framework and theoretically derive a surrogate loss that enables an implicitly hierarchical policy to learn behavior equivalent to that of an explicit hierarchical policy, leading to the proposed IH-GRPO algorithm. Extensive experiments on IH-GRPO achieve absolute improvements of 1.87\%, 2.16\%, and 2.53\% on Qwen3-1.7B, Qwen3-4B, and Qwen3-8B across six out-of-domain mathematical reasoning benchmarks over the strongest baseline method, while also yielding consistent performance gains in other domains. Our code is available at 
\url{https://anonymous.4open.science/r/IH-GRPO-172F}.
\end{abstract}

\section{Introduction}
Reinforcement learning (RL) has become a key paradigm for enhancing the reasoning capabilities of large language models (LLMs). In particular, reinforcement learning with verifiable rewards (RLVR) has successfully induced explicit reasoning behaviors~\cite{shao2024deepseekmath}. To further improve LLM reasoning, recent work has increasingly explored agentic reasoning, where LLMs interact with external tools or feedback during problem solving. Tool-integrated reasoning (TIR) instantiates this paradigm by incorporating external tools into the reasoning process, allowing LLMs to tackle a broader range of problems. In mathematical reasoning~\cite{lu2026contextual,lin2026resrl,yang2026your}, code is a powerful tool, enabling precise computation, verification, and the use of structural properties such as recursion to simplify problems and improve performance. Accordingly, recent work increasingly integrates code to strengthen mathematical reasoning in LLMs.

\begin{figure*}[t]
    \centering
    \includegraphics[width=0.95\textwidth, keepaspectratio]{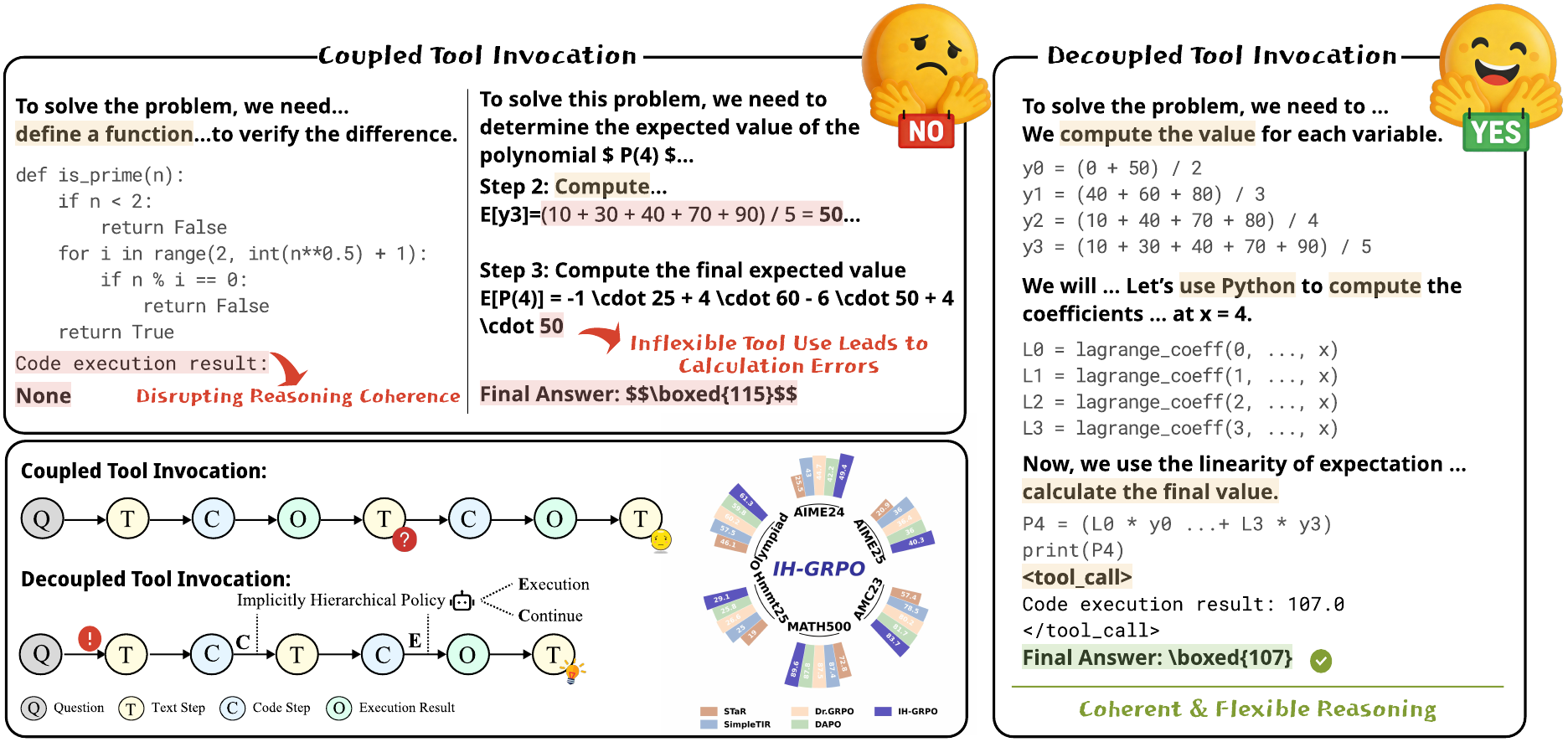}
    \caption{\textbf{(Top-left)} Coupled tool invocation triggers immediate function calls, leading to empty outputs, disrupted reasoning coherence, and premature termination due to hallucinated results. In complex calculations, manual computation is often error-prone. Rigid tool-use patterns prevent the model from flexibly leveraging code tools to handle intermediate computational steps, thereby increasing the likelihood of reasoning errors. \textbf{(Right)} Decoupled tool invocation facilitates tool-assisted intermediate steps, ensuring coherent and flexible reasoning. \textbf{(Bottom-left)} A comparison between coupled and decoupled tool invocation. IH-GRPO delivers superior performance across six mathematical reasoning datasets on Qwen3-1.7B/4B/8B.}
    \label{fig:motivation}
\end{figure*} 

Despite the benefits of code-assisted reasoning, existing work~\cite{li2025torl,xue2025simpletir,li2025teaching,wu2025templaterl} often lacks an explicit notion of when external feedback is needed, which may limit the effectiveness of TIR. Since the execution result of a code tool serves as external feedback to the model, its timing is critical: premature feedback may interrupt the model's ongoing deliberation and disrupt the coherence of intermediate reasoning. Nevertheless, existing methods typically couple tool invocation with immediate execution, without a mechanism for deciding whether a generated code fragment should be executed to elicit feedback. As illustrated in Figure~\ref{fig:motivation}, once a code block is detected in the model output, it is executed eagerly, even when the model is merely defining intermediate functions or variables rather than requesting external feedback. Moreover, tightly coupling code generation with execution implicitly requires each code block to be a relatively complete executable unit, limiting the model's ability to interleave textual reasoning across multiple code fragments. This constraint may limit the expressivity and diversity of reasoning patterns, potentially degrading overall performance. A detailed statistical analysis is provided in Section~\ref{sec:system_tool_analysis}.

To address these limitations, we propose and formalize the new problem of decoupling tool invocation from execution in TIR. As shown in Figure~\ref{fig:motivation}, all generated code is deferred by default, allowing the model to autonomously decide when to execute. This design enables flexible interleaving of textual and code-based reasoning without breaking reasoning continuity. It also allows the model to verify and reflect on generated code before execution, and enhances the model's expressivity and diversity of reasoning patterns. For additional examples, see Appendix~\ref{appendix:case_study}. Furthermore, to enable finer-grained control over tool execution, we introduce
a hierarchical control mechanism inspired by hierarchical learning and
theoretically derive a surrogate loss that enables an implicitly hierarchical
policy to emulate the training dynamics of an explicit one. By inducing
endogenous hierarchical control through an objective-level correction, our
method avoids relying on an external controller model or hand-crafted execution
workflows, offering greater flexibility and leading to the IH-GRPO algorithm. Experiments on multiple Qwen3 models show that IH-GRPO substantially improves mathematical reasoning across six benchmarks while also enhancing broader model capabilities. Our contributions can be summarized as follows:
\begin{itemize}
    \item \textbf{Decoupling tool invocation and execution:} Starting from the identified limitation in existing TIR, which is the lack of awareness of when external feedback is needed and can impair reasoning, we formalize and analyze the new problem of decoupling tool invocation and execution. To the best of our knowledge, we are the first to formalize this problem.
    
    \item \textbf{Hierarchical Policy Training via IH-GRPO:} Inspired by the success of hierarchical learning, we introduce hierarchical control and theoretically derive a surrogate loss, enabling implicit hierarchical policies to emulate explicit ones without altering the model architecture or incurring the training and inference overhead of explicit hierarchy. We further propose the IH-GRPO algorithm, which adds only a loss term to achieve more efficient hierarchical control of tool invocation and execution.
    
    \item \textbf{Strong Reasoning Performance:} We evaluate our approach across six out-of-distribution mathematical reasoning benchmarks. Compared to the strongest baseline, IH-GRPO achieves improvements of 1.87\%, 2.16\%, and 2.53\% on the Qwen3-1.7B, Qwen3-4B, and Qwen3-8B models, respectively, while also showing gains in other domains, thereby validating the effectiveness of IH-GRPO.
\end{itemize}

\section{Further Analysis on Tool Invocation}
\label{sec:system_tool_analysis}
We compare coupled and decoupled invocation by analyzing 768 randomly sampled Qwen3-8B responses from the corresponding training settings.\footnote{Statistics were obtained with GPT-5.5 and matched manual checks on a random 10\% subset.}

\textbf{Inference Coherence}: As shown in Figure~\ref{fig:motivation} (left), the coupled setting's lack of control over tool invocation timing led to 4.98\% of responses being interrupted by unexpected external returns, resulting in hallucinated outputs and incorrect answers. In contrast, the decoupled setting, with improved control over invocation timing, enhanced inference coherence and reduced interruptions to 0.68\%.

\textbf{Diversity of Tool Usage Patterns}: As shown in Figure~\ref{fig:compare_tool_usage}, in the coupled setting, each
code block is expected to form a relatively complete executable code segment. This requirement encourages the model to defer tool invocation to the end of the reasoning process, where tools are often used for redundant verification or even merely to produce the final answer, resulting in inefficient tool usage. In contrast, the decoupled setting enables the model to interleave textual and code-based reasoning more flexibly, allowing code to replace intermediate steps that involve complex computation. This shifts tool usage toward more substantive intermediate reasoning, substantially reducing redundant verification and yielding a more flexible and effective tool-use pattern.

\begin{figure}[t]
    \centering
    \includegraphics[width=1.0\columnwidth]{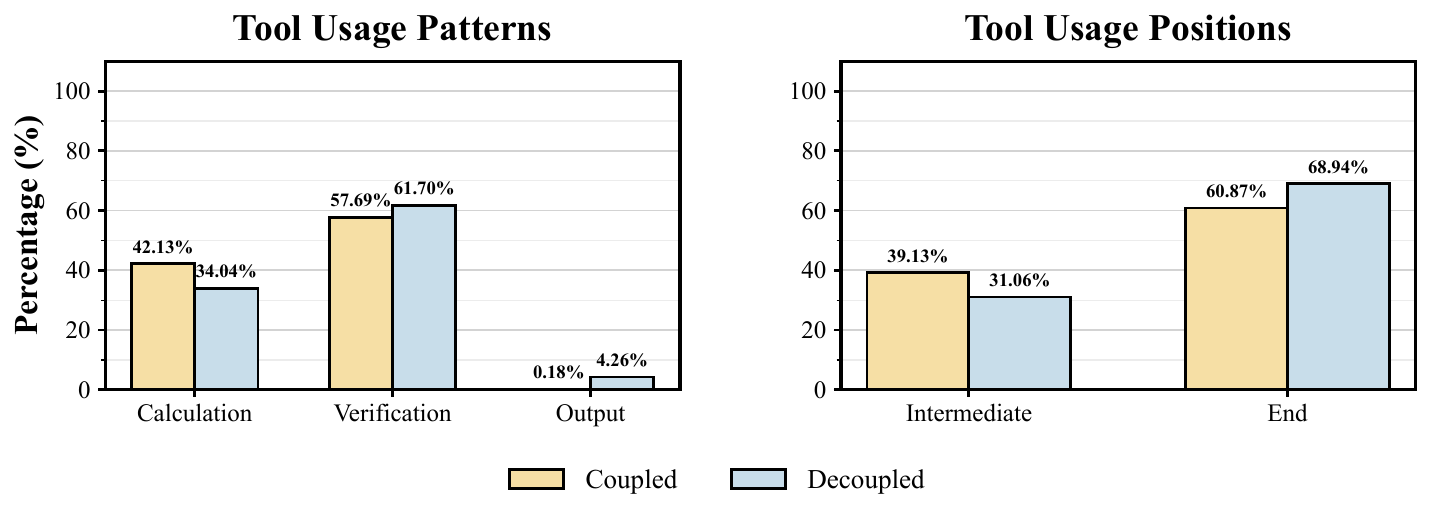}
    \caption{Comparison of invocation methods: \textbf{(Left)} tool usage patterns, \textbf{(Right)} tool positions.}
    \label{fig:compare_tool_usage}
\end{figure}

\section{Preliminary}
\subsection{Group Relative Policy Optimization}
In this section, we introduce the group relative policy optimization (GRPO)~\cite{shao2024deepseekmath} algorithm. Given a dataset \( D \) with questions \( q \) and corresponding ground-truth answers \( a \), GRPO samples rollout trajectories \( \{o_1, \dots, o_G\} \) from the old policy \( \pi_{\theta_{old}} \) and optimizes the current policy \( \pi_{\theta} \) by maximizing the following objective:
\begin{align} & \mathcal{J}_{\text{GRPO}}(\theta) = \mathbb{E}_{q \sim D, \{o_i\} \sim \pi_{\theta}} \Bigg[ \frac{1}{G} \sum_{i=1}^G \frac{1}{|o_i|} \sum_{t=1}^{|o_i|} s_{i,t} \Bigg], \nonumber \\ 
& s_{i,t} = \min \left( \rho_{i,t} A_{i,t}, \text{clip} \left( \rho_{i,t}, 1 - \epsilon, 1 + \epsilon \right) A_{i,t} \right), \nonumber \\
& \rho_{i,t} = \frac{\pi_{\theta}(o_{i,t} | q, o_{i,<t})}{\pi_{\theta_{\text{old}}}(o_{i,t} | q, o_{i,<t})}, 
\label{equ:grpo_loss} 
\end{align}
where \(\epsilon\) is a hyperparameter that controls the clipping range of the importance-sampling ratio, and the KL-penalty term is omitted in our objective. The advantage term \( A_{i,t} \) is given by:
\begin{equation}
A_{i,t} = \frac{r_i - \text{mean}(\{r_1, \dots, r_G\})}{\text{std}(\{r_1, \dots, r_G\})}
\end{equation}
Here, \( r_i \) is the reward for trajectory \( o_i \), which is evaluated by a rule-based verifier.

\subsection{Types of Reasoning Paths}
Reasoning steps can be classified into three distinct types (Figure~\ref{fig:node_type_show}):

\begin{itemize}
    \item \textbf{Sequential Step:} Execution does not affect subsequent reasoning paths and can be deferred until a later critical step.
    \item \textbf{Branching Step:} Execution determines subsequent reasoning paths and should be performed immediately.
    \item \textbf{Output Step:} Generate the final answer and should be executed immediately.
\end{itemize}

In previous code-augmented mathematical reasoning, nearly all code is executed immediately upon generation, preventing the model from regulating tool execution. However, based on the step types above, some executions can be delayed. To address this limitation, we explicitly distinguish between \emph{critical steps}, which include branching and output steps, and \emph{non-critical steps}, which are sequential steps. By deferring the execution of sequential steps and merging them with subsequent critical steps, our approach allows the model to control tool execution, thereby decoupling tool invocation from execution.

\begin{figure}[t]
    \centering
    \includegraphics[width=0.8\columnwidth]{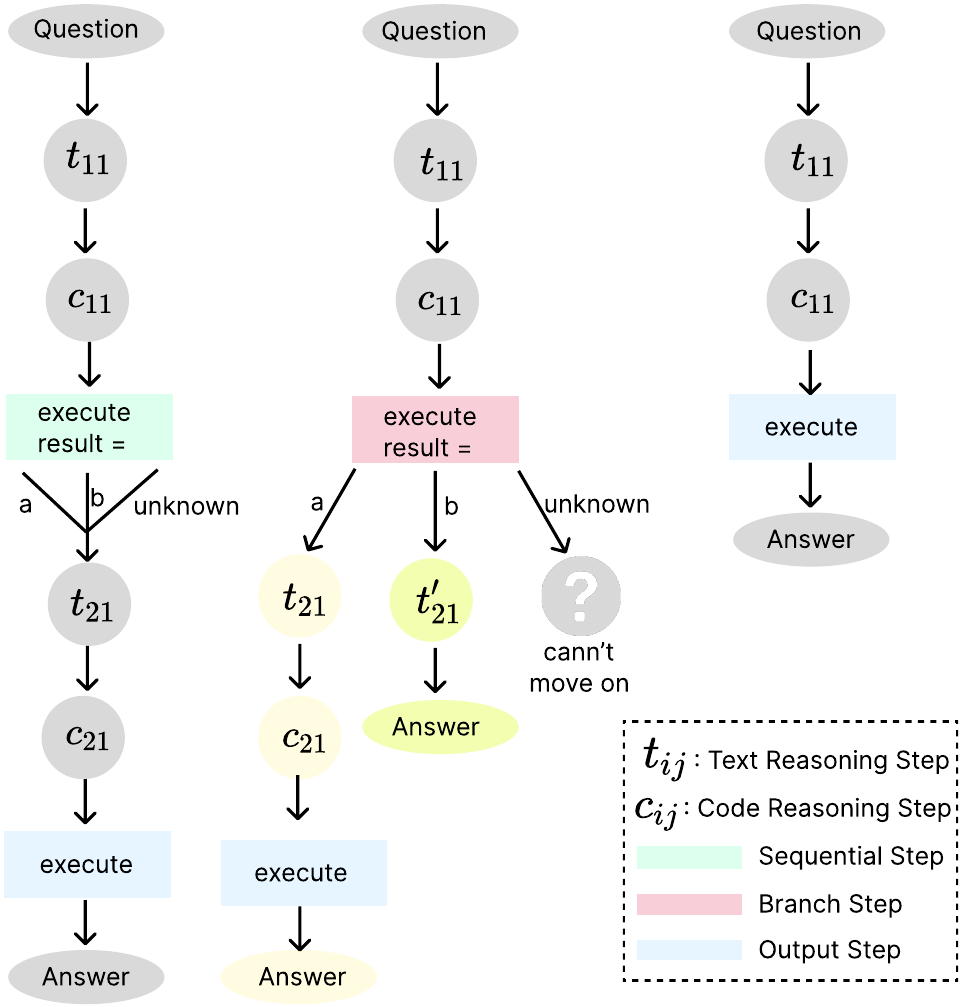}
    \caption{Different step types in the reasoning process.}
    \label{fig:node_type_show}
\end{figure}

\section{Method}
We introduce the problem of decoupling tool invocation from tool execution and
formalize it in Sec.~\ref{sec:problem_formulation}. To enable explicit control
over tool execution, we propose a hierarchical policy structure in
Sec.~\ref{sec:hie_control}. We then derive the IH-GRPO objective and discuss the
corresponding reward design in Sec.~\ref{sec:ih-grpo}. The following sections
provide a detailed description of these components.

\subsection{Problem Formulation}
\label{sec:problem_formulation}
We propose a new problem of decoupling tool invocation and execution and introduce a delayed execution mechanism with explicit control. The reasoning trajectory interleaves textual and code steps, with tool execution deferred and selectively triggered. Formally, the trajectory is defined as:
\begin{align}
\tilde{\tau} = (&t_{11}, c_{11}, \dots, e_1, o_1, \notag\\
                &\dots, t_{l i}, c_{l i}, \dots, t_{l s}, c_{l s}, e_l, o_l, \dots, t_{n1}),
\end{align}
where $n$ denotes the total number of interaction turns, and $t_{l i}$ and $c_{l i}$ represent the $i$-th interleaved text and code reasoning steps in turn $l$. The execution signal $e_l$ concatenates all previously unexecuted code blocks $\{c_{l1}, \dots, c_{l s}\}$ accumulated since the last execution signal $e_{l-1}$ into a single executable block, whose execution yields an observation $o_l$ that is appended to the context. At each step, the model conditions on the entire history of generated text, code, and observations to produce subsequent reasoning steps, repeating this process until the final turn $n$ terminates with a textual reasoning step $t_{n1}$ that contains the final answer. The likelihood of generating such a trajectory under an LLM policy $\pi_\theta$ parameterized by $\theta$ factorizes as:
\begin{align}
P_{\pi_\theta}(\tilde{\tau} \mid q) &= \prod_{l=1}^{n} \Bigg[ \Bigg[ \prod_{i=1}^{l_s} P_{\pi_\theta}\!\Big( t_{l i} \,\big|\, q, \mathcal{H}^{<t_{l i}} \Big) \notag\\
 &\quad \cdot P_{\pi_\theta}\!\Big( c_{l i} \,\big|\, t_{l i}, q, \mathcal{H}^{<c_{l i}} \Big) \Bigg] \notag\\
 &\quad \cdot P_{\pi_\theta}\!\Big( e_l \,\big|\, q, \mathcal{H}^{<e_l} \Big)\Bigg],
\end{align}
where $q$ denotes the input problem together with the model’s guiding prompt, $\mathcal{H}^{<x}$ denotes the complete history prior to step $x$.

\subsection{Hierarchical Control of Tool Execution}
\label{sec:hie_control}
Building on the decoupling of tool invocation and execution, the model is required to autonomously decide when tool execution is needed to incorporate external observations. Inspired by hierarchical control frameworks~\cite{zhang2025agent, chen2025enhancing}, we introduce a hierarchical control strategy for execution-time decision making. In this framework, a high-level policy determines whether to execute a tool, while a low-level policy governs token-level generation when the tool is not executed. Formally, let $S$ denote the current reasoning context. The high-level policy $\pi_{\text{high}}$, parameterized by $\theta_0$, controls whether to perform a real tool execution. We denote the high-level actions $a_{\text{high}}$ as $C$ (continue reasoning) and $E$ (execute), with the corresponding action distribution defined as:
\begin{equation}
\pi_{\text{high}}(a_{high} \mid S) =
\begin{cases}
\sigma(\theta_0), & \text{if } C\\
1 - \sigma(\theta_0), & \text{if } E,
\end{cases}
\end{equation}
where $\sigma(\cdot)$ denotes the sigmoid function. For the low-level policy $\pi_{\text{low}}$, when the high-level policy decides not to execute the tool, it generates the next token in an autoregressive manner. Let $\theta_i$ ($i=1,\dots,V$) denote the parameters of the low-level policy, where $V$ is the vocabulary size. The token distribution is given by:
\begin{equation}
\pi_{\text{low}}(\theta_i \mid S,a_{high}=C)
= \frac{\exp(\theta_{i})}{\sum_{u=1}^{V} \exp(\theta_{u})}, i \geq 1.
\end{equation}
Thus, the joint policy distribution of the explicit hierarchical strategy $\pi_E$ is given by:
\begin{equation}
\pi_E(\theta_i \mid S) =
\begin{cases}
\sigma(\theta_0) \cdot \frac{\exp(\theta_{i})}{\sum_{u=1}^{V} \exp(\theta_{u})}, & \text{if } C, \, i \ge 1,\\[2mm]
1 - \sigma(\theta_0), & \text{if } E.
\end{cases}
\end{equation}

However, the explicit hierarchical strategy requires joint decision-making by two policies, introducing dual control at each reasoning step and substantially increasing both training and inference costs. To alleviate this inefficiency, we propose an implicit hierarchical strategy that approximates the behavior of the explicit hierarchical model without relying on an explicit high-level policy. In an autoregressive language model, the tool-execution token is treated identically to other tokens, naturally forming an implicit hierarchy. Let the model parameters be $\{\beta_0, \beta_1, \dots, \beta_V\}$, where $\beta_0$ corresponds to the tool-execution token and $V+1$ is the vocabulary size. The implicit hierarchy policy distribution $\pi_I$ can be expressed as:
\begin{equation}
\pi_I(\beta_i \mid S) = \frac{e^{\beta_i}}{\sum_{s=0}^{V} e^{\beta_s}}
\end{equation}
Our objective is to enable the implicit hierarchical policy to learn the behavior of the explicit policy, achieving efficient training and inference without modeling the high-level policy. At initialization, there exists an explicit hierarchical policy \( \pi_E(\theta_i) \) such that the output probabilities of both policies are identical for every token:
\begin{equation}
\pi_E(\theta_i \mid S) = \pi_I(\beta_i \mid S).
\end{equation}
However, after one update, the two distributions typically diverge due to differences in control structures. To maintain equivalence during optimization, we introduce a surrogate loss that ensures the updated explicit policy \( \pi_E' \) remains consistent with the updated implicit policy \( \pi_I' \). We theoretically derive the following theorem by analyzing the policy gradient updates of both strategies.

\begin{theorem}[Equivalence of One-Step Updates Between Explicit and Implicit Hierarchical Policies]
\label{thm:equivalence}
There exists an explicit hierarchical policy $\pi_E(\theta_i)$ that is equivalent to the initial implicit hierarchical policy $\pi_I(\beta_i)$, i.e.,
\begin{equation}
\forall i, \quad \pi_E(\theta_i \mid S) = \pi_I(\beta_i \mid S)
\end{equation}
Then, under a single policy-gradient update with learning rate $\eta$, 
the equivalence between $\pi_E(\theta_i \mid S)$ and $\pi_I(\beta_i \mid S)$ can be preserved by optimizing the following 
surrogate loss $L_I'(\beta_i)$ for the implicit policy:
\begin{align}
& L_I'(\beta_i) = - A \left[ \beta_i - \log \left( \sum_{s=0}^V e^{\beta_s} \right) \right] - A \cdot \text{sg}(\gamma_i) \nonumber \\
& \quad \cdot \log Z_i + \left( A \log Z_i - f_i \cdot \beta_0 \right) \cdot \mathbb{I}\{i \geq 1\}, \nonumber \\
& Z'_i = \sum_{s=1}^V \exp\left( \beta_s + \eta A \left( \delta_{s i} - \text{softmax}_{1-V}(\beta_s) \right) \right), \nonumber \\
& Z_i = \sum_{s=1}^V e^{\beta_s}, \gamma_i = \frac{Z_i}{e^{\beta_0} + Z_i}, f_i = \frac{1}{\eta} \ln \left( \text{sg}(\frac{Z'_i}{Z_i}) \right),
\end{align}
where $A$ is the advantage and $\text{sg}(\cdot)$ denotes the stop-gradient operation.
\end{theorem}
\textbf{Remark.} The proof is provided in Appendix~\ref{appendix:imp_hie_proof}. Under this formulation, after each gradient update, $L_I'(\beta_i)$ ensures that $\pi_E'(\theta' \mid S) = \pi_I'(\beta' \mid S)$, thereby maintaining the equivalence between the updated explicit and implicit hierarchical policies throughout the training process. 

\subsection{IH-GRPO Algorithm}
\label{sec:ih-grpo}
In Theorem~\ref{thm:equivalence}, assuming small parameter updates, $Z' \approx Z$, and thus $f \approx 0$, which is reasonable in post-training RL with a small learning rate~\cite{sheng2024hybridflow}. To streamline the derivation, we omitted KL penalties, a common practice in tool-use scenarios to encourage exploration, as noted in prior works~\cite{xue2025simpletir,bai2025towards,dong2025agentic}. Our surrogate loss derivation ensures equivalence between the two hierarchical policies at each update step, making explicit consideration of clipping unnecessary.

\paragraph{IH-GRPO objective.} Without loss of generality, let the 0-th token control tool execution. Based on this, we extend the GRPO objective to incorporate the hierarchical correction term, leading to the implicit hierarchical group relative policy optimization (IH-GRPO) objective:
\begin{align}
& \mathcal{J}_{\text{IH-GRPO}}(\theta) = \mathbb{E}_{q \sim D, \{o_i\} \sim \pi_{\theta}} \Bigg[ \frac{1}{G} \sum_{i=1}^G  \frac{1}{|o_i|} \sum_{t=1}^{|o_i|} \nonumber \\
& \quad \Big( s_{i,t} + \lambda \cdot \text{sg}(s_{i,t}) \cdot c_{i,t} \Big) \Bigg], \nonumber \\
& c_{i,t} = \text{sg}(\gamma_{i,t}) \cdot \log Z_{i,t} - \log Z_{i,t} \cdot \mathbb{I}\{o_{i,t} \neq \text{0}\}, \nonumber \\
& Z_{i,t} = \sum_{s=1}^V e^{\beta_{i,t,s}}, \gamma_{i,t} = \frac{Z_{i,t}}{e^{\beta_{i,t,0}} + Z_{i,t}},
\label{equ:ih_grpo_loss}
\end{align}
where $\beta_{i,t,s}$ denotes the output logit of the LLM for the $s$-th token
in the vocabulary at position $t$ of the $i$-th rollout, $\lambda$ is a coefficient that balances the hierarchical correction term and the standard GRPO objective. By optimizing $\mathcal{J}_{\text{IH-GRPO}}$, the implicit policy is able to approximate the behavior of an explicit hierarchical strategy, thereby enabling more effective control over tool execution while avoiding the computational and inference overhead associated with training an explicit high-level policy, improving overall reasoning performance. The full algorithm is given in Appendix~\ref{appendix:pseudocode}.

\paragraph{Reward design.} We use an outcome correctness reward to encourage correct answers by evaluating whether the model’s output matches the ground truth:
\begin{equation}
R_{\text{correct}} =
\begin{cases}
1, & \text{if match},\\
0, & \text{otherwise}.
\end{cases}
\end{equation}

\begin{table*}[!ht]
\centering
\caption{Evaluation results of different methods. `D` denotes decoupled tool invocation, and `C` denotes coupled tool invocation. The \textbf{bold} and \underline{underline} indicate the best and second-best results, respectively.}
\small
\renewcommand{\arraystretch}{1.1}
\setlength{\tabcolsep}{0pt}
\begin{tabular*}{\textwidth}{@{\extracolsep{\fill}}lcccccccc}
\toprule
\textbf{Method} & \textbf{Tool} & \textbf{AIME24} & \textbf{AIME25} & \textbf{MATH500} & \textbf{AMC23} & \textbf{Hmmt25} & \textbf{Olympiad} & \textbf{Average} \\
\midrule
\rowcolor{gray!8}\multicolumn{9}{c}{\textit{Models based on Qwen3-1.7B}} \\
\midrule
Base & No & 22.50 & 21.77 & 81.05 & 57.89 & 9.58 & 43.30 & 39.35 \\
STaR & C & 11.25 & 12.92 & 64.72 & 43.36 & 9.27 & 35.93 & 29.57 \\
STaR & D & 11.67 & 11.46 & 58.82 & 38.12 & 8.44 & 32.55 & 26.84 \\
SimpleTIR & C & 26.67 & 24.06 & \textbf{83.97} & 60.70 & 9.69 & 46.30 & 41.90 \\
SimpleTIR & D & 29.27 & \underline{27.19} & 81.70 & 63.67 & 16.56 & 49.18 & 44.60 \\
Dr.GRPO & C & 27.92 & 25.42 & 79.80 & 64.38 & \underline{18.33} & \textbf{51.57} & 44.57 \\
Dr.GRPO & D & \underline{29.37} & 26.25 & 82.20 & 62.66 & 16.98 & 51.13 & \underline{44.77} \\
DAPO & C & 25.73 & 26.56 & 83.75 & \textbf{67.73} & 14.06 & 50.77 & \underline{44.77} \\
DAPO & D & 24.90 & 23.54 & 78.78 & 60.08 & 15.83 & 46.55 & 41.61 \\
EH-GRPO & D & 27.81 & 22.60 & 82.02 & 64.06 & 17.40 & 50.17 & 44.01 \\
\rowcolor[RGB]{236,244,252}
\textbf{IH-GRPO (ours)} & D & \textbf{31.15} & \textbf{27.50} & \underline{83.80} & \underline{66.87} & \textbf{18.96} & \underline{51.55} & \textbf{46.64} \\
\midrule
\rowcolor{gray!8}\multicolumn{9}{c}{\textit{Models based on Qwen3-4B}} \\
\midrule
Base & No & 34.17 & 24.69 & 86.35 & 70.70 & 12.50 & 47.33 & 45.96 \\
STaR & C & 26.35 & 22.60 & 73.77 & 61.09 & 20.62 & 48.93 & 42.23 \\
STaR & D & 32.40 & 23.65 & 78.47 & 68.12 & 24.69 & 49.72 & 46.18 \\
SimpleTIR & C & 49.27 & \underline{43.75} & 87.88 & 85.31 & 31.67 & 60.93 & 59.80 \\
SimpleTIR & D & \underline{51.98} & 41.25 & \underline{91.33} & 89.14 & 32.29 & 63.60 & \underline{61.60} \\
Dr.GRPO & C & 49.69 & 42.71 & 90.30 & 87.11 & 30.42 & 63.90 & 60.69 \\
Dr.GRPO & D & 51.67 & 43.65 & 88.70 & 86.02 & \textbf{35.73} & 63.17 & 61.49 \\
DAPO & C & 49.38 & 42.60 & 91.15 & 88.83 & 31.87 & \underline{64.57} & 61.40 \\
DAPO & D & 50.10 & 40.83 & 91.10 & \underline{89.30} & 30.63 & 63.15 & 60.85 \\
EH-GRPO & D & 41.77 & 35.31 & 90.25 & 83.28 & 26.98 & 61.92 & 56.59 \\
\rowcolor[RGB]{236,244,252}
\textbf{IH-GRPO (ours)} & D & \textbf{55.31} & \textbf{45.83} & \textbf{91.45} & \textbf{90.00} & \underline{35.10} & \textbf{64.88} & \textbf{63.76} \\
\midrule
\rowcolor{gray!8}\multicolumn{9}{c}{\textit{Models based on Qwen3-8B}} \\
\midrule
Base & No & 34.69 & 23.65 & 84.80 & 66.80 & 11.04 & 46.47 & 44.58 \\
STaR & C & 38.85 & 27.29 & 79.85 & 67.73 & 26.98 & 53.55 & 49.04 \\
STaR & D & 40.83 & 25.00 & 80.55 & 67.42 & 28.65 & 52.32 & 49.13 \\
SimpleTIR & C & 52.92 & 40.31 & 90.25 & 89.53 & \textbf{33.54} & 65.28 & 61.97 \\
SimpleTIR & D & \underline{58.33} & \underline{44.69} & 92.52 & 91.02 & 31.35 & 64.57 & \underline{63.75} \\
Dr.GRPO & C & 56.35 & 41.15 & 92.50 & 89.14 & 31.04 & 65.07 & 62.54 \\
Dr.GRPO & D & 57.50 & 41.35 & 92.70 & 90.55 & 31.98 & \underline{66.50} & 63.43 \\
DAPO & C & 51.56 & 38.75 & 88.52 & 88.67 & 31.46 & 64.07 & 60.51 \\
DAPO & D & 56.15 & 43.23 & 92.35 & \underline{91.56} & 32.81 & 64.50 & 63.43 \\
EH-GRPO & D & 55.42 & 43.02 & \underline{93.40} & 89.14 & 27.81 & 64.10 & 62.15 \\
\rowcolor[RGB]{236,244,252}
\textbf{IH-GRPO (ours)} & D & \textbf{61.67} & \textbf{47.50} & \textbf{93.50} & \textbf{94.30} & \underline{33.13} & \textbf{67.58} & \textbf{66.28} \\
\bottomrule
\end{tabular*}
\label{tab:qwen3-main-acc}
\end{table*}

\section{Experiments}
\subsection{Experiment Setup}
\paragraph{Tasks and Datasets} We focus on mathematical reasoning tasks and enhance the model’s capabilities through code tools, employing a sandboxed Jupyter runtime that preserves intermediate states for multi-turn reasoning. Training data consist of Math3-5 from SimpleRL~\cite{zeng2025simplerl} and Deepscaler~\cite{luo3deepscaler}. Detailed dataset statistics are provided in Appendix~\ref{appendix:dataset_details}.

\paragraph{Baselines} We compare our approach with the following baselines: (i) the CoT distillation-based method STaR~\cite{zelikman2022star}; (ii) RL-based methods, including SimpleTIR~\cite{xue2025simpletir}, Dr.GRPO~\cite{liu2025understanding}, and DAPO~\cite{yu2025dapo}; and (iii) the explicit hierarchical method EH-GRPO integrates an external model to control tool execution and uses GRPO for training. Experiments were conducted for all methods under both coupled (C) and decoupled (D) tool invocation settings, except for the hierarchical approaches (EH, IH), which require control over tool execution only in the decoupled tool invocation scenario.

\paragraph{Implementation} We use the Qwen-3 family~\cite{yang2025qwen3} as base models, including Qwen3-1.7B, Qwen3-4B, and Qwen3-8B. To facilitate more efficient reasoning, we employ a ``no thinking'' mode. Training uses a maximum of 5 interaction rounds, a maximum response length of 8K tokens, and a maximum prompt length of 16K tokens. For all RL-based methods, following SimpleTIR, we employ a void-turn\footnote{A void turn refers to a turn in which the model neither invokes any tool nor produces a final answer.} filter along with oversampling, low-advantage filtering, and data difficulty control to enhance training stability and improve performance. Further implementation details, including baseline setups and hyperparameter settings, are provided in Appendix~\ref{appendix:exp_details}.

\paragraph{Evaluation} We evaluate all methods on out-of-domain mathematical reasoning benchmarks, including MATH500~\cite{hendrycks2021measuring}, AIME24, AIME25, AMC23, Hmmt Feb 25, and Olympiad~\cite{he2024olympiadbench}. Following SimpleTIR~\cite{xue2025simpletir}, we set the sampling temperature to 1.0 and use average@8 for evaluation, applying average@32 for smaller datasets (AIME, AMC, and Hmmt) to improve stability. Answer correctness is automatically verified with \texttt{math-verify}.

\begin{figure*}[t]
    \centering
    \begin{minipage}{0.28\textwidth}
        \centering
        \includegraphics[width=\textwidth]{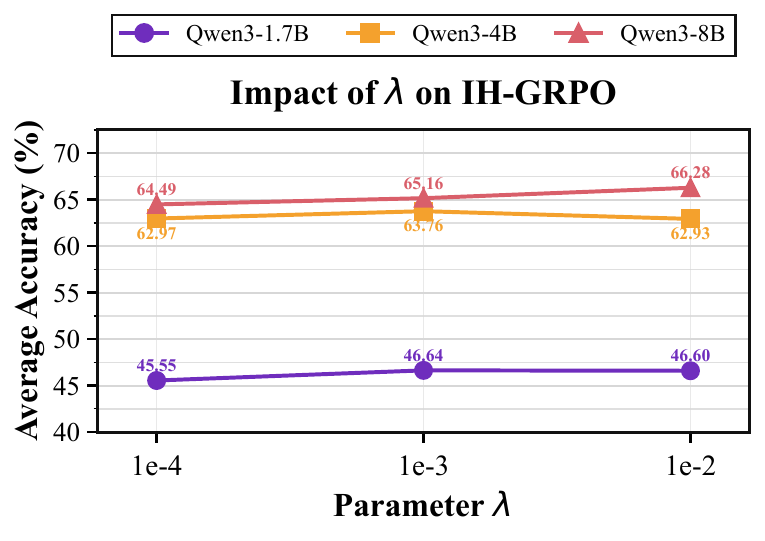}
    \end{minipage}
    \begin{minipage}{0.58\textwidth}
        \centering
        \includegraphics[width=\textwidth]{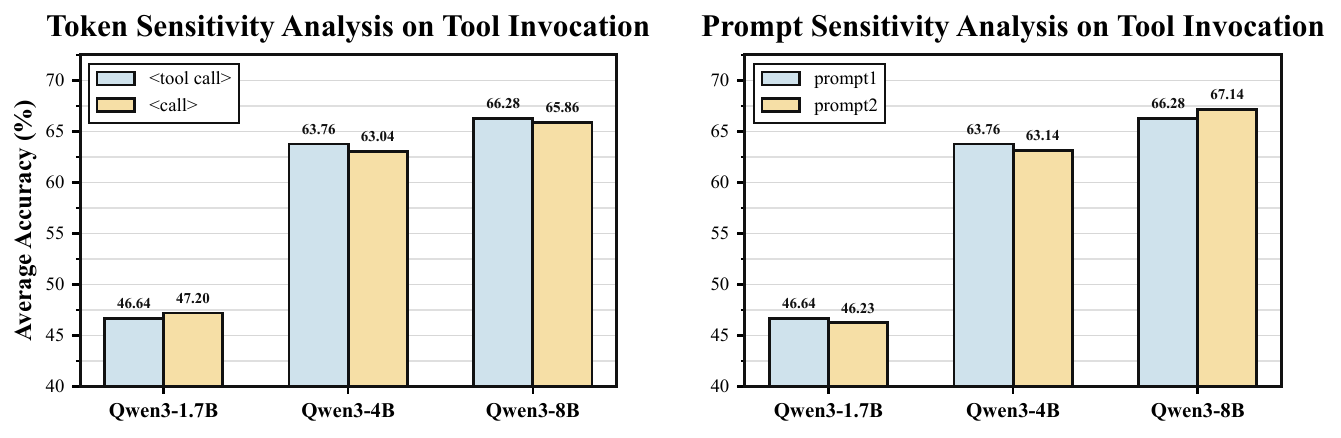}
    \end{minipage}
    \caption{\textbf{(Left)} Performance of IH-GRPO across varying $\lambda$. \textbf{(Middle)} Token sensitivity and \textbf{(Right)} Prompt sensitivity analysis of IH-GRPO on Qwen3 models.}
    \label{fig:combined_analysis}
\end{figure*}

\subsection{Main Results}
The results in Table~\ref{tab:qwen3-main-acc} demonstrate that IH-GRPO outperforms the strongest baseline by 1.87\%, 2.16\%, and 2.53\% on Qwen3-1.7B, Qwen3-4B, and Qwen3-8B, respectively. STaR shows inferior performance, sometimes even underperforming the base model, due to frequent tool-use errors from insufficient interaction with the tool environment. In contrast, RL-based methods consistently perform better, highlighting the importance of training models to interact with tools. EH-GRPO incurs additional overhead from an external controller and relies on insufficient hidden-layer representations, leading to suboptimal performance. IH-GRPO, however, excels in reasoning performance, empirically validating its effectiveness. After TIR training with IH-GRPO, the model shows improvements across various domains, as detailed in Appendix~\ref{sec:ood_evaluation}.

\subsection{Ablation Studies}
\label{sec:ablation_study}
\paragraph{Validation of component design.} For most other models and RL algorithms, the decoupled setting consistently outperforms the coupled setting in most cases, validating the effectiveness of decoupled invocation. IH-GRPO, differing from SimpleTIR(D) only in the loss function, achieves higher accuracy, confirming the benefit of the surrogate loss. These results highlight the contribution of each component in IH-GRPO.

\paragraph{Hyperparameter Analysis of $\lambda$.} Figure~\ref{fig:combined_analysis} (left) shows that IH-GRPO is robust to the
surrogate-loss weight $\lambda$ over $10^{-4}$--$10^{-2}$. Qwen3-1.7B and
Qwen3-4B perform best at $\lambda=10^{-3}$, while Qwen3-8B favors
$\lambda=10^{-2}$, possibly due to its greater robustness in leveraging
stronger tool-execution signals. Overall, all settings remain competitive,
indicating low sensitivity to $\lambda$.

\paragraph{Tool Token Sensitivity.} Figure~\ref{fig:combined_analysis} (middle) shows that IH-GRPO is robust to the choice
of tool-execution tokens. We compare the Qwen3-specific \texttt{<tool\_call>}
token with the generic marker \texttt{<call>}, using the core token
\texttt{call} to compute the surrogate loss for more stable tokenization.
Across all three model scales, performance differences are marginal, with a
slight gain on 1.7B.

\paragraph{Prompt Sensitivity Analysis.} As shown in Figure~\ref{fig:combined_analysis} (right), IH-GRPO remains stable across
the two prompts detailed in Appendix~\ref{appendix:prompt_details}, suggesting
that its gains mainly come from hierarchical execution control and the proposed
objective design rather than prompt engineering.

\begin{figure}[t]
    \centering
    \includegraphics[width=1.0\columnwidth]{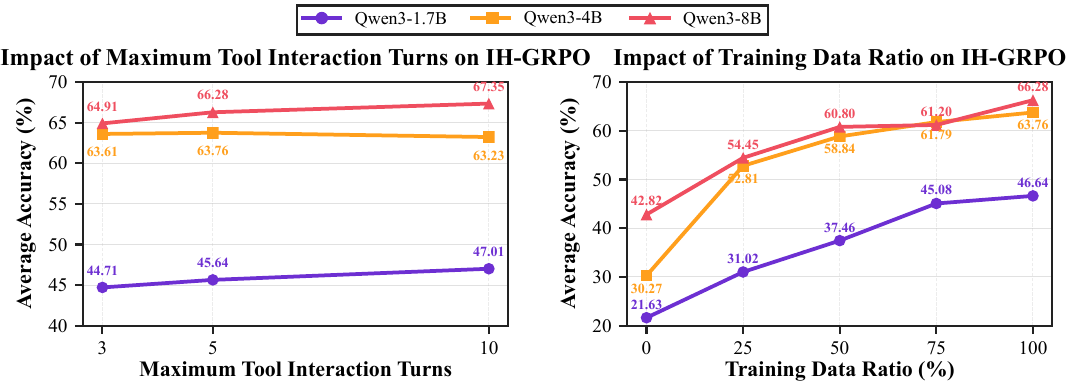}
    \caption{Performance of IH-GRPO under (\textbf{left}) varying tool-interaction budgets  and (\textbf{right}) training-data ratios.}
    \label{fig:more_analysis}
\end{figure}

\paragraph{More Tool Interaction Turns.} As shown in Fig.~\ref{fig:more_analysis} (left), increasing the maximum number of tool
executions from 3 to 10 generally improves performance, as larger budgets
provide more opportunities for tool-assisted reasoning. Notably, all three
models remain stable even with 10 executions and achieve further gains,
indicating that IH-GRPO is robust to larger tool-interaction budgets.

\paragraph{Impact of Varying Training Data Sizes.} As shown in Figure~\ref{fig:more_analysis} (right), IH-GRPO improves steadily as the
training data ratio increases. With limited data, models struggle to interact
effectively with the tool-augmented environment and underperform the base models
in Table~\ref{tab:qwen3-main-acc}. Increasing the data ratio improves tool
invocation and yields the best performance with the full dataset.

\section{Related work}
\subsection{Tool-Integrated Mathematical Reasoning}
Recent work has extensively explored enhancing LLMs' ability to leverage external tools~\cite{jin2025search,song2025r1,bai2025towards, luo2025agent}. In mathematical reasoning, code-based tools significantly enhance models' expressive and computational capacities, improving reasoning performance~\cite{lin2025understanding}. Early approaches~\cite{gou2023tora,huang2025distilling} typically rely on supervised fine-tuning (SFT) or direct preference optimization (DPO)~\cite{rafailov2023direct} to teach tool usage. ReAct~\cite{yao2022react} integrates reasoning with action, while Chain-of-Code~\cite{li2023chain} allows LLMs to execute a combination of pseudocode and executable code. To advance TIR, some studies incorporate RL after SFT~\cite{feng2025retool,shang2025rstar2, luo2025agentmath}, or use RL from scratch~\cite{lin2025awpo, wu2026atlas,lin2025rest}, enabling direct interaction with external tools. SimpleTIR~\cite{xue2025simpletir} mitigates gradient instability by filtering out trajectories with void turns. Despite these efforts, most methods lack awareness of when external feedback is needed and tightly couple tool invocation with execution, potentially undermining reasoning performance. In contrast, we introduce delayed execution with explicit control, leveraging RL to enhance TIR performance using a single model.

\subsection{Hierarchical Learning}
Hierarchical learning is an important aspect of human cognition~\cite{murray2014hierarchy,huntenburg2018large}, separating high-level planning from low-level execution to support more effective problem solving. This principle underlies hierarchical designs in machine learning, including hierarchical reinforcement learning (HRL)~\cite{pateria2021hierarchical, nachum2018data, vezhnevets2017feudal}, which decomposes long-horizon tasks to ease low-level policy learning~\cite{nachum2019does}, and hierarchical structures in LLMs~\cite{wan2025rema, wang2025hierarchical, learningemergent,wu2024beyond} for multi-step reasoning. In tool-integrated reasoning (TIR), hierarchy is often introduced to separate planning from execution~\cite{lu2025octotools,liu2024tool}. For example, ReWOO~\cite{xu2023rewoo} first generates a high-level plan and then executes tools for subtasks in isolated contexts, while CodeSteer~\cite{chen2025codesteer} coordinates text and code reasoning through multi-model collaboration. However, these methods rely on externally imposed hierarchies, such as
prompt-based planning, external models, or hand-crafted workflows, which may
hurt reasoning continuity and scalability. In contrast, IH-GRPO learns
hierarchical control within a single model and implicitly realizes planning and
execution through a theoretically derived surrogate objective, without external
intervention.

\section{Conclusion}
In this work, we identify a key limitation in code-assisted mathematical reasoning: the lack of self-awareness about when external feedback is necessary, which can impact performance. To address this, we propose and formalize the new problem of decoupling tool invocation and execution, introducing delayed execution with explicit control, and conduct a thorough statistical analysis. We further propose a hierarchical control strategy and theoretically derive a surrogate loss that enables implicitly hierarchical policies to emulate the behavior of explicit hierarchical policies, resulting in the IH-GRPO algorithm. Extensive experiments on Qwen3 models ranging from 1.7B to 8B parameters across six out-of-distribution mathematical reasoning benchmarks demonstrate consistent improvements over strong baselines, while also enhancing performance in other domains.

\section*{Limitations}
Our experiments focus on relatively small-scale models with 1.7B to 8B parameters. While these results provide empirical evidence supporting the effectiveness of IH-GRPO under the considered settings, we do not evaluate the method on larger-scale models due to computational constraints. We adopt the Qwen3 model family in the non-thinking mode as the backbone, motivated by its favorable inference efficiency and solid base performance. More comprehensive evaluations on additional model families and different model variants and configurations are left for future work to further investigate the generality and scalability of IH-GRPO across diverse backbone models.

Moreover, our experimental evaluation is primarily confined to mathematical reasoning tasks and environments augmented with code tools, which are widely used benchmarks in prior work. However, the core idea of our method, which decouples tool invocation and execution to improve reasoning coherence and introduces a theoretical framework for implicit hierarchical control, can potentially generalize to broader domains and multi-tool settings. For example, in multi-tool scenarios, decoupling tool invocation from execution allows the output of one tool to be directly reused as the input to another. By controlling when execution occurs, this paradigm reduces unnecessary feedback and context length, while enabling flexible interleaving between textual reasoning and tool-based computation. Importantly, it avoids external model intervention, prompt-based coordination, and manually specified workflows. As a result, this design helps preserve coherent reasoning, offers greater flexibility, and may further improve TIR performance.

\section*{Ethical Statements}
Our experiments rely exclusively on publicly available datasets; these tasks involve reasoning and multi-turn tool execution, without involving any sensitive user data. Our methods focus on improving models' reasoning and tool execution capabilities, and any deployment in open-ended or real-world settings should be carefully monitored to mitigate potential risks.


\bibliography{custom}

@article{xue2025simpletir,
  title={Simpletir: End-to-end reinforcement learning for multi-turn tool-integrated reasoning},
  author={Xue, Zhenghai and Zheng, Longtao and Liu, Qian and Li, Yingru and Zheng, Xiaosen and Ma, Zejun and An, Bo},
  journal={arXiv preprint arXiv:2509.02479},
  year={2025}
}

@article{wu2025templaterl,
  title={TemplateRL: Structured Template-Guided Reinforcement Learning for LLM Reasoning},
  author={Wu, Jinyang and Liao, Chonghua and Feng, Mingkuan and Zhang, Shuai and Wen, Zhengqi and Luo, Haoran and Yang, Ling and Xu, Huazhe and Tao, Jianhua},
  journal={arXiv preprint arXiv:2505.15692},
  year={2025}
}

@article{lin2025understanding,
  title={Understanding tool-integrated reasoning},
  author={Lin, Heng and Xu, Zhongwen},
  journal={arXiv preprint arXiv:2508.19201},
  year={2025}
}

@article{shang2025rstar2,
  title={rstar2-agent: Agentic reasoning technical report},
  author={Shang, Ning and Liu, Yifei and Zhu, Yi and Zhang, Li Lyna and Xu, Weijiang and Guan, Xinyu and Zhang, Buze and Dong, Bingcheng and Zhou, Xudong and Zhang, Bowen and others},
  journal={arXiv preprint arXiv:2508.20722},
  year={2025}
}

@article{feng2025retool,
  title={Retool: Reinforcement learning for strategic tool use in llms},
  author={Feng, Jiazhan and Huang, Shijue and Qu, Xingwei and Zhang, Ge and Qin, Yujia and Zhong, Baoquan and Jiang, Chengquan and Chi, Jinxin and Zhong, Wanjun},
  journal={arXiv preprint arXiv:2504.11536},
  year={2025}
}

@article{li2025torl,
  title={Torl: Scaling tool-integrated rl},
  author={Li, Xuefeng and Zou, Haoyang and Liu, Pengfei},
  journal={arXiv preprint arXiv:2503.23383},
  year={2025}
}

@article{bai2025towards,
  title={Towards Effective Code-Integrated Reasoning},
  author={Bai, Fei and Min, Yingqian and Zhang, Beichen and Chen, Zhipeng and Zhao, Wayne Xin and Fang, Lei and Liu, Zheng and Wang, Zhongyuan and Wen, Ji-Rong},
  journal={arXiv preprint arXiv:2505.24480},
  year={2025}
}

@article{luo2025agent,
  title={Agent lightning: Train any ai agents with reinforcement learning},
  author={Luo, Xufang and Zhang, Yuge and He, Zhiyuan and Wang, Zilong and Zhao, Siyun and Li, Dongsheng and Qiu, Luna K and Yang, Yuqing},
  journal={arXiv preprint arXiv:2508.03680},
  year={2025}
}

@article{jin2025search,
  title={Search-r1: Training llms to reason and leverage search engines with reinforcement learning},
  author={Jin, Bowen and Zeng, Hansi and Yue, Zhenrui and Yoon, Jinsung and Arik, Sercan and Wang, Dong and Zamani, Hamed and Han, Jiawei},
  journal={arXiv preprint arXiv:2503.09516},
  year={2025}
}

@article{song2025r1,
  title={R1-searcher: Incentivizing the search capability in llms via reinforcement learning},
  author={Song, Huatong and Jiang, Jinhao and Min, Yingqian and Chen, Jie and Chen, Zhipeng and Zhao, Wayne Xin and Fang, Lei and Wen, Ji-Rong},
  journal={arXiv preprint arXiv:2503.05592},
  year={2025}
}

@article{dong2025agentic,
  title={Agentic reinforced policy optimization},
  author={Dong, Guanting and Mao, Hangyu and Ma, Kai and Bao, Licheng and Chen, Yifei and Wang, Zhongyuan and Chen, Zhongxia and Du, Jiazhen and Wang, Huiyang and Zhang, Fuzheng and others},
  journal={arXiv preprint arXiv:2507.19849},
  year={2025}
}

@article{zhang2025agent,
  title={Agent-as-Tool: A Study on the Hierarchical Decision Making with Reinforcement Learning},
  author={Zhang, Yanfei},
  journal={arXiv preprint arXiv:2507.01489},
  year={2025}
}

@article{chen2025enhancing,
  title={Enhancing LLM-Based Agents via Global Planning and Hierarchical Execution},
  author={Chen, Junjie and Li, Haitao and Yang, Jingli and Liu, Yiqun and Ai, Qingyao},
  journal={arXiv preprint arXiv:2504.16563},
  year={2025}
}

@article{sheng2024hybridflow,
  title   = {HybridFlow: A Flexible and Efficient RLHF Framework},
  author  = {Guangming Sheng and Chi Zhang and Zilingfeng Ye and Xibin Wu and Wang Zhang and Ru Zhang and Yanghua Peng and Haibin Lin and Chuan Wu},
  year    = {2024},
  journal = {arXiv preprint arXiv: 2409.19256}
}

@article{zeng2025simplerl,
  title={Simplerl-zoo: Investigating and taming zero reinforcement learning for open base models in the wild},
  author={Zeng, Weihao and Huang, Yuzhen and Liu, Qian and Liu, Wei and He, Keqing and Ma, Zejun and He, Junxian},
  journal={arXiv preprint arXiv:2503.18892},
  year={2025}
}

@article{luo3deepscaler,
  title={Deepscaler: Surpassing o1-preview with a 1.5 b model by scaling rl, 2025},
  author={Luo, Michael and Tan, Sijun and Wong, Justin and Shi, Xiaoxiang and Tang, William Y and Roongta, Manan and Cai, Colin and Luo, Jeffrey and Zhang, Tianjun and Li, Li Erran and others},
  journal={Notion Blog},
  volume={3},
  number={4},
  pages={5}
}

@article{yang2025qwen3,
  title={Qwen3 technical report},
  author={Yang, An and Li, Anfeng and Yang, Baosong and Zhang, Beichen and Hui, Binyuan and Zheng, Bo and Yu, Bowen and Gao, Chang and Huang, Chengen and Lv, Chenxu and others},
  journal={arXiv preprint arXiv:2505.09388},
  year={2025}
}

@article{shao2024deepseekmath,
  title={Deepseekmath: Pushing the limits of mathematical reasoning in open language models},
  author={Shao, Zhihong and Wang, Peiyi and Zhu, Qihao and Xu, Runxin and Song, Junxiao and Bi, Xiao and Zhang, Haowei and Zhang, Mingchuan and Li, YK and Wu, Yang and others},
  journal={arXiv preprint arXiv:2402.03300},
  year={2024}
}

@article{pateria2021hierarchical,
  title={Hierarchical reinforcement learning: A comprehensive survey},
  author={Pateria, Shubham and Subagdja, Budhitama and Tan, Ah-hwee and Quek, Chai},
  journal={ACM Computing Surveys (CSUR)},
  volume={54},
  number={5},
  pages={1--35},
  year={2021},
  publisher={ACM New York, NY, USA}
}

@article{nachum2019does,
  title={Why does hierarchy (sometimes) work so well in reinforcement learning?},
  author={Nachum, Ofir and Tang, Haoran and Lu, Xingyu and Gu, Shixiang and Lee, Honglak and Levine, Sergey},
  journal={arXiv preprint arXiv:1909.10618},
  year={2019}
}

@article{nachum2018data,
  title={Data-efficient hierarchical reinforcement learning},
  author={Nachum, Ofir and Gu, Shixiang Shane and Lee, Honglak and Levine, Sergey},
  journal={Advances in neural information processing systems},
  volume={31},
  year={2018}
}

@inproceedings{vezhnevets2017feudal,
  title={Feudal networks for hierarchical reinforcement learning},
  author={Vezhnevets, Alexander Sasha and Osindero, Simon and Schaul, Tom and Heess, Nicolas and Jaderberg, Max and Silver, David and Kavukcuoglu, Koray},
  booktitle={International conference on machine learning},
  pages={3540--3549},
  year={2017},
  organization={PMLR}
}

@article{learningemergent,
  title={Emergent Hierarchical Reasoning in LLMs through Reinforcement Learning},
  author={LEARNING, THROUGH REINFORCEMENT}
}

@article{wang2025hierarchical,
  title={Hierarchical Reasoning Model},
  author={Wang, Guan and Li, Jin and Sun, Yuhao and Chen, Xing and Liu, Changling and Wu, Yue and Lu, Meng and Song, Sen and Yadkori, Yasin Abbasi},
  journal={arXiv preprint arXiv:2506.21734},
  year={2025}
}

@article{wan2025rema,
  title={Rema: Learning to meta-think for llms with multi-agent reinforcement learning},
  author={Wan, Ziyu and Li, Yunxiang and Wen, Xiaoyu and Song, Yan and Wang, Hanjing and Yang, Linyi and Schmidt, Mark and Wang, Jun and Zhang, Weinan and Hu, Shuyue and others},
  journal={arXiv preprint arXiv:2503.09501},
  year={2025}
}

@article{lu2025octotools,
  title={Octotools: An agentic framework with extensible tools for complex reasoning},
  author={Lu, Pan and Chen, Bowen and Liu, Sheng and Thapa, Rahul and Boen, Joseph and Zou, James},
  journal={arXiv preprint arXiv:2502.11271},
  year={2025}
}

@article{liu2024tool,
  title={Tool-planner: Dynamic solution tree planning for large language model with tool clustering},
  author={Liu, Yanming and Peng, Xinyue and Zhang, Yuwei and Cao, Jiannan and Zhang, Xuhong and Cheng, Sheng and Wang, Xun and Yin, Jianwei and Du, Tianyu},
  journal={arXiv e-prints},
  pages={arXiv--2406},
  year={2024}
}

@article{zelikman2022star,
  title={Star: Bootstrapping reasoning with reasoning},
  author={Zelikman, Eric and Wu, Yuhuai and Mu, Jesse and Goodman, Noah},
  journal={Advances in Neural Information Processing Systems},
  volume={35},
  pages={15476--15488},
  year={2022}
}

@article{yu2025dapo,
  title={Dapo: An open-source llm reinforcement learning system at scale},
  author={Yu, Qiying and Zhang, Zheng and Zhu, Ruofei and Yuan, Yufeng and Zuo, Xiaochen and Yue, Yu and Dai, Weinan and Fan, Tiantian and Liu, Gaohong and Liu, Lingjun and others},
  journal={arXiv preprint arXiv:2503.14476},
  year={2025}
}

@inproceedings{kwon2023efficient,
  title={Efficient memory management for large language model serving with pagedattention},
  author={Kwon, Woosuk and Li, Zhuohan and Zhuang, Siyuan and Sheng, Ying and Zheng, Lianmin and Yu, Cody Hao and Gonzalez, Joseph and Zhang, Hao and Stoica, Ion},
  booktitle={Proceedings of the 29th symposium on operating systems principles},
  pages={611--626},
  year={2023}
}

@article{wang2024mmlu,
  title={Mmlu-pro: A more robust and challenging multi-task language understanding benchmark},
  author={Wang, Yubo and Ma, Xueguang and Zhang, Ge and Ni, Yuansheng and Chandra, Abhranil and Guo, Shiguang and Ren, Weiming and Arulraj, Aaran and He, Xuan and Jiang, Ziyan and others},
  journal={Advances in Neural Information Processing Systems},
  volume={37},
  pages={95266--95290},
  year={2024}
}

@article{liu2020logiqa,
  title={Logiqa: A challenge dataset for machine reading comprehension with logical reasoning},
  author={Liu, Jian and Cui, Leyang and Liu, Hanmeng and Huang, Dandan and Wang, Yile and Zhang, Yue},
  journal={arXiv preprint arXiv:2007.08124},
  year={2020}
}

@article{suzgun2022challenging,
  title={Challenging BIG-Bench Tasks and Whether Chain-of-Thought Can Solve Them},
  author={Suzgun, Mirac and Scales, Nathan and Sch{\"a}rli, Nathanael and Gehrmann, Sebastian and Tay, Yi and Chung, Hyung Won and Chowdhery, Aakanksha and Le, Quoc V and Chi, Ed H and Zhou, Denny and and Wei, Jason},
  journal={arXiv preprint arXiv:2210.09261},
  year={2022}
}

@article{hendrycks2021measuring,
  title={Measuring mathematical problem solving with the math dataset},
  author={Hendrycks, Dan and Burns, Collin and Kadavath, Saurav and Arora, Akul and Basart, Steven and Tang, Eric and Song, Dawn and Steinhardt, Jacob},
  journal={arXiv preprint arXiv:2103.03874},
  year={2021}
}

@inproceedings{he2024olympiadbench,
  title={Olympiadbench: A challenging benchmark for promoting agi with olympiad-level bilingual multimodal scientific problems},
  author={He, Chaoqun and Luo, Renjie and Bai, Yuzhuo and Hu, Shengding and Thai, Zhen and Shen, Junhao and Hu, Jinyi and Han, Xu and Huang, Yujie and Zhang, Yuxiang and others},
  booktitle={Proceedings of the 62nd Annual Meeting of the Association for Computational Linguistics (Volume 1: Long Papers)},
  pages={3828--3850},
  year={2024}
}

@article{li2025teaching,
  title={Teaching Language Models to Reason with Tools},
  author={Li, Chengpeng and Tang, Zhengyang and Li, Ziniu and Xue, Mingfeng and Bao, Keqin and Ding, Tian and Sun, Ruoyu and Wang, Benyou and Wang, Xiang and Lin, Junyang and others},
  journal={arXiv preprint arXiv:2510.20342},
  year={2025}
}

@article{gou2023tora,
  title={Tora: A tool-integrated reasoning agent for mathematical problem solving},
  author={Gou, Zhibin and Shao, Zhihong and Gong, Yeyun and Shen, Yelong and Yang, Yujiu and Huang, Minlie and Duan, Nan and Chen, Weizhu},
  journal={arXiv preprint arXiv:2309.17452},
  year={2023}
}

@article{huang2025distilling,
  title={Distilling Tool Knowledge into Language Models via Back-Translated Traces},
  author={Huang, Xingyue and Hu, Xianglong and Ding, Zifeng and He, Yuan and Alzarooni, Waleed and Ye, Ziyu and Fan, Wendong and He, Bailan and Bo, Haige and Hu, Changran and others},
  journal={arXiv preprint arXiv:2506.19171},
  year={2025}
}

@article{luo2025agentmath,
  title={AgentMath: Empowering Mathematical Reasoning for Large Language Models via Tool-Augmented Agent},
  author={Luo, Haipeng and Feng, Huawen and Sun, Qingfeng and Xu, Can and Zheng, Kai and Wang, Yufei and Yang, Tao and Hu, Han and Tang, Yansong and Wang, Di},
  journal={arXiv preprint arXiv:2512.20745},
  year={2025}
}

@article{liu2025understanding,
  title={Understanding r1-zero-like training: A critical perspective},
  author={Liu, Zichen and Chen, Changyu and Li, Wenjun and Qi, Penghui and Pang, Tianyu and Du, Chao and Lee, Wee Sun and Lin, Min},
  journal={arXiv preprint arXiv:2503.20783},
  year={2025}
}

@article{murray2014hierarchy,
  title={A hierarchy of intrinsic timescales across primate cortex},
  author={Murray, John D and Bernacchia, Alberto and Freedman, David J and Romo, Ranulfo and Wallis, Jonathan D and Cai, Xinying and Padoa-Schioppa, Camillo and Pasternak, Tatiana and Seo, Hyojung and Lee, Daeyeol and others},
  journal={Nature neuroscience},
  volume={17},
  number={12},
  pages={1661--1663},
  year={2014},
  publisher={Nature Publishing Group US New York}
}

@article{huntenburg2018large,
  title={Large-scale gradients in human cortical organization},
  author={Huntenburg, Julia M and Bazin, Pierre-Louis and Margulies, Daniel S},
  journal={Trends in cognitive sciences},
  volume={22},
  number={1},
  pages={21--31},
  year={2018},
  publisher={Elsevier}
}

@article{wu2024beyond,
  title={Beyond examples: High-level automated reasoning paradigm in in-context learning via mcts},
  author={Wu, Jinyang and Feng, Mingkuan and Zhang, Shuai and Che, Feihu and Wen, Zengqi and Liao, Chonghua and Tao, Jianhua},
  journal={arXiv preprint arXiv:2411.18478},
  year={2024}
}

@article{wu2026atlas,
  title={Atlas: Orchestrating Heterogeneous Models and Tools for Multi-Domain Complex Reasoning},
  author={Wu, Jinyang and Zhai, Guocheng and Jin, Ruihan and Yuan, Jiahao and Shen, Yuhao and Zhang, Shuai and Wen, Zhengqi and Tao, Jianhua},
  journal={arXiv preprint arXiv:2601.03872},
  year={2026}
}

@article{xu2023rewoo,
  title={Rewoo: Decoupling reasoning from observations for efficient augmented language models},
  author={Xu, Binfeng and Peng, Zhiyuan and Lei, Bowen and Mukherjee, Subhabrata and Liu, Yuchen and Xu, Dongkuan},
  journal={arXiv preprint arXiv:2305.18323},
  year={2023}
}

@article{chen2025codesteer,
  title={CodeSteer: Symbolic-Augmented Language Models via Code/Text Guidance},
  author={Chen, Yongchao and Hao, Yilun and Liu, Yueying and Zhang, Yang and Fan, Chuchu},
  journal={arXiv preprint arXiv:2502.04350},
  year={2025}
}

@article{li2023chain,
  title={Chain of code: Reasoning with a language model-augmented code emulator},
  author={Li, Chengshu and Liang, Jacky and Zeng, Andy and Chen, Xinyun and Hausman, Karol and Sadigh, Dorsa and Levine, Sergey and Fei-Fei, Li and Xia, Fei and Ichter, Brian},
  journal={arXiv preprint arXiv:2312.04474},
  year={2023}
}

@inproceedings{yao2022react,
  title={React: Synergizing reasoning and acting in language models},
  author={Yao, Shunyu and Zhao, Jeffrey and Yu, Dian and Du, Nan and Shafran, Izhak and Narasimhan, Karthik R and Cao, Yuan},
  booktitle={The eleventh international conference on learning representations},
  year={2022}
}

@article{rafailov2023direct,
  title={Direct preference optimization: Your language model is secretly a reward model},
  author={Rafailov, Rafael and Sharma, Archit and Mitchell, Eric and Manning, Christopher D and Ermon, Stefano and Finn, Chelsea},
  journal={Advances in neural information processing systems},
  volume={36},
  pages={53728--53741},
  year={2023}
}

@article{lu2026contextual,
  title={Contextual Rollout Bandits for Reinforcement Learning with Verifiable Rewards},
  author={Lu, Xiaodong and Wang, Xiaohan and Chai, Jiajun and Yin, Guojun and Lin, Wei and Chen, Zhijun and Luo, Yu and Zhuang, Fuzhen and Ban, Yikun and Wang, Deqing},
  journal={arXiv preprint arXiv:2602.08499},
  year={2026}
}

@article{lin2026resrl,
  title={ResRL: Boosting LLM Reasoning via Negative Sample Projection Residual Reinforcement Learning},
  author={Lin, Zihan and Wang, Xiaohan and Cao, Jie and Chai, Jiajun and Wang, Li and Lu, Xiaodong and Lin, Wei and He, Ran and Yin, Guojun},
  journal={arXiv preprint arXiv:2605.00380},
  year={2026}
}

@article{yang2026your,
  title={Your Group-Relative Advantage Is Biased},
  author={Yang, Fengkai and Chen, Zherui and Wang, Xiaohan and Lu, Xiaodong and Chai, Jiajun and Yin, Guojun and Lin, Wei and Ma, Shuai and Zhuang, Fuzhen and Wang, Deqing and others},
  journal={arXiv preprint arXiv:2601.08521},
  year={2026}
}

@article{lin2025awpo,
  title={AWPO: Enhancing Tool-Use of Large Language Models through Adaptive Integration of Reasoning Rewards},
  author={Lin, Zihan and Wang, Xiaohan and Yang, Hexiong and Chai, Jiajun and Cao, Jie and Yin, Guojun and Lin, Wei and He, Ran},
  journal={arXiv preprint arXiv:2512.19126},
  year={2025}
}

@article{lin2025rest,
  title={ResT: Reshaping Token-Level Policy Gradients for Tool-Use Large Language Models},
  author={Lin, Zihan and Wang, Xiaohan and Cao, Jie and Chai, Jiajun and Yin, Guojun and Lin, Wei and He, Ran},
  journal={arXiv preprint arXiv:2509.21826},
  year={2025}
}

\appendix
\renewcommand{\thetable}{A\arabic{table}} 
\renewcommand{\thefigure}{A\arabic{figure}} 
\renewcommand{\thetheorem}{A\arabic{theorem}} 
\setcounter{table}{0} 
\setcounter{figure}{0} 
\setcounter{theorem}{0}

\onecolumn
\section{Implicit Hierarchical Derivation}
\label{appendix:imp_hie_proof}
\begin{theorem}[Equivalence of One-Step Updates Between Explicit and Implicit Hierarchical Policies]
There exists an explicit hierarchical policy $\pi_E(\theta_i \mid S)$ that is equivalent to the initial implicit hierarchical policy $\pi_I(\beta_i \mid S)$, i.e.,
\[
\forall i, \quad \pi_E(\theta_i \mid S) = \pi_I(\beta_i \mid S).
\]
Then, under a single policy-gradient update with learning rate $\eta$, 
the equivalence between $\pi_E(\theta_i \mid S)$ and $\pi_I(\beta_i \mid S)$ can be preserved by optimizing the following 
surrogate loss $L_I'(\beta_i)$ for the implicit policy:
\begin{align}
& L_I'(\beta_i) = - A \left[ \beta_i - \log \left( \sum_{s=0}^V e^{\beta_s} \right) \right] - A \cdot \text{sg}(\gamma_i) \cdot \log Z_i + \left( A \log Z_i - f_i \cdot \beta_0 \right) \cdot \mathbb{I}\{i \geq 1\}, \nonumber \\
& Z_i = \sum_{s=1}^V e^{\beta_s}, \gamma_i = \frac{Z_i}{e^{\beta_0} + Z_i}, f_i = \frac{1}{\eta} \ln \left( \text{sg}(\frac{Z'_i}{Z_i}) \right), Z'_i = \sum_{s=1}^V \exp\left( \beta_s + \eta A \left( \delta_{s i} - \text{softmax}_{1-V}(\beta_s) \right) \right),
\end{align}
where $A$ is the advantage and $\text{sg}(\cdot)$ denotes the stop-gradient operation.
\end{theorem}

\begin{proof}
We define the explicit hierarchical strategy as a set of parameters \(\{\theta_0, \theta_1, \dots, \theta_V\}\), where the high-level actions are denoted by \(C\) (continue reasoning) and \(E\) (execution). The probability of \(\theta_i\) conditioned on the state \(S\) is given by:
\[
P\left(\theta_i \mid S\right) = 
\begin{cases} 
\sigma\left(\theta_0\right) \cdot \text{sofmax}_{1-V}\left(\theta_i\right), & \text{if } C, i \geq 1, \\
1 - \sigma\left(\theta_0\right), & \text{if } E.
\end{cases}
\]
Next, we define the implicit hierarchical strategy as a set of parameters \(\left\{\beta_0, \beta_1, \cdots, \beta_V\right\}\). The probability of \(\beta_i\) conditioned on the state \(S\) is:

\[
p\left(\beta_i \mid S\right) = \frac{e^{\beta_i}}{\sum_{s=0}^V e^{\beta_s}}.
\]

\subsection{Step 1: Initial Condition}
We assume:
\[
P(\theta_i | S) = P(\beta_i | S),
\]
\[
\Rightarrow 1 - \sigma(\theta_0) = \frac{e^{\beta_0}}{\sum_{s=0}^V e^{\beta_s}} 
\quad \implies \quad \sigma(\theta_0) = \frac{\sum_{s=1}^V e^{\beta_s}}{\sum_{s=0}^V e^{\beta_s}}.
\]
Substituting $\sigma(\theta_0)$ into $P(\theta_i \mid S)$ for case $E$, we obtain:
\[
\frac{\sum_{s=1}^V e^{\beta_s}}{\sum_{s=0}^V e^{\beta_s}} \cdot \frac{e^{\theta_i}}{\sum_{s=1}^V e^{\theta_s}} 
= \frac{e^{\beta_i}}{\sum_{s=0}^V e^{\beta_s}}
\Rightarrow \quad 
\frac{e^{\beta_i}}{\sum_{s=1}^V e^{\beta_s}} = \frac{e^{\theta_i}}{\sum_{s=1}^V e^{\theta_s}}.
\]
We derive the equivalent conditions of $P(\theta_i | S) = P(\beta_i | S)$:
\[
1.\ \sigma(\theta_0) = \frac{\sum_{s=1}^V e^{\beta_s}}{\sum_{s=0}^V e^{\beta_s}} = \gamma \quad \text{and} \quad 2.\ \frac{e^{\beta_i}}{\sum_{s=1}^V e^{\beta_s}} = \frac{e^{\theta_i}}{\sum_{s=1}^V e^{\theta_i}}.
\]
Without loss of generality, assume \( \theta_i = \beta_i \) for \( i \geq 1 \) from condition \(\boldsymbol{2}\). Besides, we have: \( \theta_0 = \ln\sum_{s=1}^V e^{\beta_s} - \beta_0 \). Therefore, \( \{\beta_0, \beta_1, \dots, \beta_V\} \) can equivalently represent \( \{\theta_0, \theta_1, \dots, \theta_V\} \) from the initial condition.

\subsection{Step 2: Explicit Hierarchical Policy Update}
We assume use policy gradient loss: 
\begin{align}
L_{E}(\theta_i) = -A \cdot \log(P(\theta_i|S)).
\end{align}
When updating token $ i \neq 0 $: \\
\[
\frac{\partial L_{E}(\theta_i)}{\partial \theta_j} = -\frac{\partial A \cdot \left[ \log(\sigma(\theta_0)) + \theta_i - \log\sum_{s=1}^V e^{\theta_s} \right]}{\partial \theta_j}
=
\begin{cases} 
-A \cdot \left[ 1 - \text{softmax}_{1-V}(\theta_j) \right], & j = i \text{ and } j \geq 1 \\
-A \cdot \left[ -\text{softmax}_{1-V}(\theta_j) \right], & j \neq i \text{ and } j \geq 1, 
\end{cases}
\]
\[
\frac{\partial L_{E}(\theta_i)}{\partial \theta_0} = -A(1 - \sigma(\theta_0)),
\]
\[
\frac{\partial L_{E}(\theta_i)}{\partial \theta_j} =
\begin{cases} 
- A(1 - \sigma(\theta_0)), & j = 0 \\
A \cdot \text{softmax}_{1-V}(\theta_j), & j \neq i \text{ and } j \geq 1 \\
A \cdot \text{softmax}_{1-V}(\theta_j) - A, & j = i \text{ and } j \geq 1.
\end{cases}
\]
When updating token $ i = 0 $: \\
\[
\frac{\partial L_{E}(\theta_i)}{\partial \theta_j} = - \frac{\partial A \cdot \log\left[ 1 - \sigma(\theta_0) \right]}{\partial \theta_j} =
\begin{cases} 
0, & j \neq 0 \\
A \sigma(\theta_0), & j = 0. 
\end{cases}
\]

\subsection{Step 3: Implicit Hierarchical Policy Update}
We assume use policy gradient loss: 
\begin{align}
L_{I}(\beta_i) = - A \cdot \log(P(\beta_i|S)),
\end{align}
\[
\frac{\partial L_{I}(\beta_i)}{\partial \beta_j} = - \frac{\partial A \cdot \log\frac{e^{\beta_i}}{\sum_{s=0}^V e^{\beta_s}}}{\partial \beta_j} = - \frac{\partial A \cdot \left[ \beta_i - \log\sum_{s=0}^V e^{\beta_s} \right]}{\partial \beta_j}
=
\begin{cases} 
A \cdot \text{softmax}_{0-V}(\beta_j), & j \neq i \\
A \cdot \text{softmax}_{0-V}(\beta_j) - A, & j = i 
\end{cases}
\]
\[
=
\begin{cases} 
A \cdot \text{softmax}_{1-V}(\beta_j) \cdot \sigma(\theta_0), & j \neq i \\
A \cdot \text{softmax}_{1-V}(\beta_j) \cdot \sigma(\theta_0) - A, & j = i. 
\end{cases}
\]

\subsection{Step 4: Surrogate Loss of Implicit Hierarchical Policy}
\textbf{Case 1: Sampled Token is Tool Execution (\(i=0\))}

The surrogate loss is:
\[
L_I'(\beta_i) = - A \log \left( \frac{e^{\beta_0}}{e^{\beta_0} + Z_i} \right) - A \cdot \text{sg}(\gamma_i) \cdot \log Z_i,
\]
where \(Z_i = \sum_{s=1}^V e^{\beta_s}\) and \(\gamma_i = \frac{Z_i}{e^{\beta_0} + Z_i} = \sigma(\theta_0)\).
Gradients:

\noindent The base term is \(-A (\beta_0 - \log(e^{\beta_0} + Z_i))\), yielding:
\[
\frac{\partial}{\partial \beta_0} \left[- A (\beta_0 - \log(e^{\beta_0} + Z_i)) \right] = - A \left( 1 - \frac{e^{\beta_0}}{e^{\beta_0} + Z_i} \right) = - A \gamma_i,
\]
\[
\frac{\partial}{\partial \beta_j} \left[ - A (\beta_0 - \log(e^{\beta_0} + Z_i)) \right] =  A \frac{e^{\beta_j}}{e^{\beta_0} + Z_i} = A \cdot \gamma_i \cdot \text{softmax}_{1-V}(\beta_j), \quad j \geq 1.
\]
The correction term \( - A \cdot \text{sg}(\gamma_i) \cdot \log Z_i\) contributes:
\[
\frac{\partial}{\partial \beta_0} ( - A \text{sg}(\gamma_i) \cdot \log Z_i) = 0,
\]
\[
\frac{\partial}{\partial \beta_j} ( - A \cdot \text{sg}(\gamma_i) \cdot \log Z_i) =  - A \cdot \text{sg}(\gamma_i) \cdot \frac{e^{\beta_j}}{Z_i} =  - A \cdot \text{sg}(\gamma_i) \cdot \text{softmax}_{1-V}(\beta_j), \quad j \geq 1.
\]
Overall:
\[
\frac{\partial L_I'(\beta_i)}{\partial \beta_0} = - A \gamma_i = - A \sigma(\theta_0), \quad \frac{\partial L_I'(\beta_i)}{\partial \beta_j} = 0 \quad (j \geq 1).
\]
So we ensure identical updates:
\[
\beta_0' = \beta_0 + \eta \cdot A \sigma(\theta_0), \quad \beta_j' = \beta_j, \quad j \geq 1.
\]
where $\eta$ is learning rate, and the explicit updates remain:
\[
\theta_0' = \theta_0 - \eta A \sigma(\theta_0), \quad \theta_j' = \theta_j, \quad j \geq 1,
\]
Verification of conditions:

\noindent Condition 2 holds exactly since \(\beta_j' = \beta_j = \theta_j = \theta_j'\) for \(j \geq 1\). For condition 1, let \(Z'_i = \sum_{s=1}^V e^{\beta_s'} = Z_i\) (unchanged). Then
  \[
  \gamma' = \frac{Z'_i}{e^{\beta_0'} + Z'_i} = \frac{Z_i}{e^{\beta_0 + \eta A \sigma(\theta_0)} + Z_i}.
  \]
From the initial relation \(\theta_0 = \ln Z_i - \beta_0\), we have
 \begin{align*}
    \theta_0' &= \theta_0 - \eta A \sigma(\theta_0) = \ln Z_i - \beta_0 - \eta A \sigma(\theta_0) = \ln Z_i - (\beta_0 + \eta A \sigma(\theta_0)) = \ln Z_i - \beta_0'.
\end{align*}
Thus, \(e^{\theta_0'} = Z_i e^{-\beta_0'}\), \(e^{-\theta_0'} = e^{\beta_0'}/Z_i\), and
\[
\sigma(\theta_0') = \frac{1}{1 + e^{-\theta_0'}} = \frac{Z_i}{Z_i + e^{\beta_0'}} = \gamma_i',
\]
so condition 1 holds exactly.

\textbf{Case 2: Sampled Token is Non-Tool ($i \geq 1$)}

The surrogate loss is:
\[
L_I'(\beta_i) = - A \left[ \beta_i - \log \left( \sum_{s=0}^V e^{\beta_s} \right) \right] + A (1 - \text{sg}(\gamma_i)) \cdot \log Z_i - f_i \cdot \beta_0,
\]
where $\eta$ is learning rate, $f_i = \frac{1}{\eta} \ln \left( \text{sg}(\frac{Z'_i}{Z_i}) \right)$, $Z_i = \sum_{s=1}^V e^{\beta_s}$, $Z'_i = \sum_{s=1}^V \exp\left( \beta_s + \eta A \left( \delta_{s i} - \text{softmax}_{1-V}(\beta_s) \right) \right)$, $\gamma_i = \frac{Z_i}{e^{\beta_0} + Z_i}$, and $\delta_{s i}$ denotes the Kronecker delta function, which is defined as:
\[
\delta_{s i} = 
\begin{cases} 
1 & \text{if } s = i, \\
0 & \text{otherwise}.
\end{cases}
\]
Gradients:

\noindent Let $N = \sum_{s=0}^V e^{\beta_s}$. The base term is $-A (\beta_i - \log N)$, yielding:
\[
\frac{\partial}{\partial \beta_j} \left[ - A (\beta_i - \log N) \right] = - A \left( \delta_{j i} - \text{softmax}_{0-V}(\beta_j) \right).
\]
The correction term $A (1 - \text{sg}(\gamma_i)) \cdot \log Z_i$ contributes $0$ for $j=0$ and:
\[
\frac{\partial}{\partial \beta_j} \left[ A (1 - \text{sg}(\gamma_i)) \cdot \log Z_i \right] = A (1 - \text{sg}(\gamma_i)) \cdot \text{softmax}_{1-V}(\beta_j), \quad j \geq 1.
\]
Overall:
\[
\frac{\partial L_I'(\beta_i)}{\partial \beta_j} = -A \left[ \delta_{j i} - \gamma_i \cdot \text{softmax}_{1-V}(\beta_j) \right] + A (1 - \gamma_i) \cdot \text{softmax}_{1-V}(\beta_j) = - A \left[ \delta_{j i} - \text{softmax}_{1-V}(\beta_j) \right], \quad j \geq 1,
\]
\[
\frac{\partial L_I'(\beta_i)}{\partial \beta_0} = A (1 - \gamma_i) - f_i.
\]
Updated Parameters:

\noindent Assuming gradient descent updates $\beta_i \leftarrow \beta_i - \eta \nabla L_I'(\beta_i)$:
\[
\beta_j' = \beta_j + \eta A \left( \delta_{j i} - \text{softmax}_{1-V}(\beta_j) \right) \quad (j \geq 1),
\]
\[
\beta_0' = \beta_0 - \eta \left( A (1 - \gamma_i) - f_i \right) = \beta_0 - \eta A (1 - \gamma_i) + \eta f_i.
\]
Substituting $f_i$:
\[
\beta_0' = \beta_0 - \eta A (1 - \gamma_i) + \ln \left( \frac{Z'_i}{Z_i} \right).
\]
From the explicit loss $L_E(\theta_i)$, the updates are
\[
\theta_j' = \theta_j + \eta A \left( \delta_{j i} - \text{softmax}_{1-V}(\theta_j) \right) \quad (j \geq 1),
\]
\[
\theta_0' = \theta_0 + \eta A (1 - \gamma_i).
\]
Verification of Conditions:

\noindent Condition 2 holds exactly because $\beta_j = \theta_j$ initially for $j \geq 1$ and the gradients and updates for $j \geq 1$ are identical in both policies, so $\beta_j' = \theta_j'$ for $j \geq 1$. For condition 1, recall $\theta_0 = \ln Z_i - \beta_0$ and $\gamma_i = \sigma(\theta_0) = \frac{Z_i}{e^{\beta_0} + Z_i}$. After updates:
\[
\theta_0' = \theta_0 + \eta A (1 - \gamma_i) = \ln Z_i - \beta_0 + \eta A (1 - \gamma_i).
\]
For the implicit side, $Z_i' = \sum_{s=1}^V e^{\beta_s'} = \sum_{s=1}^V e^{\theta_s'}$, and the desired relation is $\theta_0' = \ln Z_i' - \beta_0'$. Substituting:
\[
\begin{aligned}
\ln Z_i' - \beta_0' &= \ln Z_i' - \left( \beta_0 + \ln \left( \frac{Z_i'}{Z_i} \right) - \eta A (1 - \gamma_i) \right) \\
&= \ln Z_i' - \beta_0 - \ln Z_i' + \ln Z_i + \eta A (1 - \gamma_i) = \ln Z_i - \beta_0 + \eta A (1 - \gamma_i) = \theta_0',
\end{aligned}
\]
so $\theta_0' = \ln Z_i' - \beta_0'$ holds exactly. Thus, $\gamma_i' = \sigma(\theta_0') = \frac{Z_i'}{e^{\beta_0'} + Z_i'}$, satisfying condition 1 strictly.

\textbf{Summary:}

The surrogate loss function \(L_I'(\beta_i)\) for the implicit hierarchical policy is defined as follows:
\[
L_I'(\beta_i) = \begin{cases} 
- A \left[ \beta_0 - \log \left( \sum_{s=0}^V e^{\beta_s} \right) \right] - A \cdot \text{sg}(\gamma_i) \cdot \log Z_i, & \text{if } i=0 \text{ (E)}, \\
- A \left[ \beta_i - \log \left( \sum_{s=0}^V e^{\beta_s} \right) \right] - A \cdot \text{sg}(\gamma_i) \cdot \log Z_i + A \log Z_i - f_i \cdot \beta_0, & \text{if } i \geq 1 \text{ (C)},
\end{cases}
\]
\[
= - A \left[ \beta_i - \log \left( \sum_{s=0}^V e^{\beta_s} \right) \right] - A \cdot \text{sg}(\gamma_i) \cdot \log Z_i 
+ \left( A \log Z_i - f_i \cdot \beta_0 \right) \cdot \mathbb{I}\{i \geq 1\},
\]
where $\text{sg}(\cdot)$ denotes the stop-gradient operation, $\eta$ is learning rate, \(Z_i = \sum_{s=1}^V e^{\beta_s}\), \(\gamma_i = \frac{Z_i}{e^{\beta_0} + Z_i}\), \(Z_i' = \sum_{s=1}^V \exp\left( \beta_s + \eta A \left( \delta_{s i} - \text{softmax}_{1-V}(\beta_s) \right) \right)\), $f_i = \frac{1}{\eta} \ln \left( \text{sg}(\frac{Z_i'}{Z_i}) \right)$.

When employing the surrogate loss \(L_I'(\beta_i)\), a single gradient descent step for the implicit hierarchical policy yields parameters that produce an identical probability distribution to that of the explicit hierarchical policy updated with the policy gradient loss \(L_E(\theta_i)\). This preserves the initial conditions (Condition 1: \(\sigma(\theta_0) = \gamma_i = \frac{\sum_{s=1}^V e^{\beta_s}}{\sum_{s=0}^V e^{\beta_s}}\); Condition 2: \(\frac{e^{\beta_i}}{\sum_{s=1}^V e^{\beta_s}} = \frac{e^{\theta_i}}{\sum_{s=1}^V e^{\theta_s}}\)) exactly. By iterating this proof at each step, it ensures the probabilities remain equal in the next step, thereby guaranteeing that the gradient updates at each step maintain the same probability. This loss function can then implicitly train the explicit hierarchical strategy.
\end{proof}
\twocolumn

\section{Experimantal Details}
\label{appendix:exp_details}
\subsection{Datasets for IH-GRPO Main Experiments}
\label{appendix:dataset_details}
We summarize the number of samples in the training and evaluation datasets in Table~\ref{tab:datasets_statics}. Following SimpleTIR~\cite{zeng2025simplerl}, we mixed the Math3-5 and Deepscaler datasets for training. For the 1.7B model, which requires longer training to learn effective tool usage, we trained for 500 steps to ensure convergence. For the 4B and 8B models, which already exhibit stronger tool usage capabilities, we trained for 300 steps. The data used for training the reinforcement learning (RL) methods is shown in 1.7B-RL and 4/8B-RL. For the STaR method, we use the same data as in RL training and perform rollouts of 8 responses and filter the correct ones for supervised fine-tuning (SFT) training, applying this procedure in both decoupled and coupled tool invocation scenarios. The specific number of samples used for STaR training is shown in 1.7B-I/D-STaR, 4B-I/D-STaR and 8B-I/D-STaR.

\begin{table}[t]
\centering
\caption{Statistics of the training and test datasets used by all methods. `D` denotes decoupled tool invocation, and `C` denotes coupled tool invocation. The term "Logical Deduction" refers to "Logical Deduction Seven Objects".}
\label{tab:datasets_statics}
\small
\setlength{\tabcolsep}{4pt}
\renewcommand{\arraystretch}{1.1}
\begin{tabular}{lll}
\toprule
\textbf{Dataset} & \textbf{Description
} & \textbf{\# Train / Test} \\
\midrule
Math3-5            & Math Reasoning     & 8523 / - \\
Deepscaler            & Math Reasoning     & 40315 / - \\
1.7B-I-STaR            & Math Reasoning     & 10757 / - \\
1.7B-D-STaR            & Math Reasoning     & 13632 / - \\
4B-I-STaR            & Math Reasoning     & 6511 / - \\
4B-D-STaR            & Math Reasoning     & 10531 / - \\
8B-I-STaR            & Math Reasoning     & 6780 / - \\
8B-D-STaR            & Math Reasoning     & 10945 / - \\
1.7B-RL            & Math Reasoning     & 22328 / - \\
4/8B-RL            & Math Reasoning     & 13797 / - \\
AIME24   & Math Reasoning     & - / 30 \\
AIME25   & Math Reasoning     & - / 30 \\
MATH500            & Math Reasoning     & - / 500 \\
AMC23            & Math Reasoning     & - / 40 \\
Hmmt25          & Math Reasoning     & - / 30 \\
Olympiad           & Math Reasoning     & - / 674 \\
MMLU-pro           & General Ability     & - / 12032 \\
Logiqa           & Logical Reasoning     & - / 651 \\
Date Understanding           & Date Understanding     & - / 250 \\
Formal Fallacies           &  Logical Reasoning     & - / 250 \\
Logical Deduction           &  Logical Reasoning     & - / 250 \\
\bottomrule
\end{tabular}
\end{table}

\subsection{Training and Comparison Details}
\label{appendix:train_details}
Our framework is implemented on top of VeRL~\cite{sheng2024hybridflow} and SimpleTIR~\cite{xue2025simpletir}. For the STaR method, we sample 8 responses per prompt and filter the correct chain-of-thought (CoT) responses for SFT, using a learning rate of \(1 \times 10^{-5}\), training for 3 epochs, and a micro\_batch\_size\_per\_gpu of 2. The dataset distribution is summarized in Table~\ref{tab:datasets_statics}. For all RL-based methods, we adopt  a batch size of 16, a ppo\_mini\_batch\_size of 128, and a learning rate of \(1 \times 10^{-6}\). Due to these shared training strategies, the DAPO baseline behaves similarly to SimpleTIR. For EH-GRPO, we extract the final hidden layer from the reasoning model after each code block as input features and train a multi-layer perceptron agent using the REINFORCE algorithm to make tool execution decisions, employing the same reward for high-level control policy. For IH-GRPO, we use the \texttt{<tool\_call>} token to represent tool executions and set $\lambda = 1\times10^{-3}$ for Qwen3-1.7B and Qwen3-4B, and $\lambda = 1\times10^{-2}$ for Qwen3-8B. All other hyperparameters are kept consistent with the default settings of SimpleTIR. All experiments were conducted on 8$\times$H20 GPUs with 141GB of memory, using Python 3.10 and PyTorch 2.6. All models are optimized with the AdamW optimizer (\(\beta_1 = 0.9\), \(\beta_2 = 0.95\), weight decay 0.01) and accelerated via vLLM~\cite{kwon2023efficient}.

\subsection{Comparison of Computational Cost}
To compare the training overhead of different methods, we evaluate their total training cost using Qwen3-4B as the backbone model on an 8$\times$H20 GPU setup with 141\,GB total memory, as summarized in Table~\ref{tab:time_cost}. The STaR method requires additional time for preparing training data; owing to its reliance on SFT, its overall training time is relatively short, but the resulting performance gains are also limited. In contrast, RL-based methods generally achieve stronger performance improvements, at the expense of substantially longer training time. The EH-GRPO approach incurs fewer tool executions and shorter training time due to suboptimal coordination between the external high-level policy controller and low-level policy, which in turn leads to inferior performance. By comparison, IH-GRPO decouples tool invocation and execution, enriches tool-use patterns, and improves TIR capability. This may increase interaction turns and overall training time, but brings significant performance gains and achieves the best performance, which is consistent with our design goal.

\begin{table}[t]
\centering
\caption{Estimated computational time (hours) for various methods with Qwen3-4B. 'C' and 'D' represent coupled and decoupled tool invocation, respectively.}
\label{tab:time_cost}
\small
\setlength{\tabcolsep}{4pt}
\renewcommand{\arraystretch}{1.1}
\begin{tabular}{lcccc}
\toprule
\textbf{Method} & \textbf{Tool} & \textbf{Data Processing} & \textbf{Training} & \textbf{Total} \\
\midrule
STaR        & C & 9.4 & 3.4 & 12.8 \\
STaR        & D & 11.6 & 5.5 & 17.1 \\
SimpleTIR  & C & 0  & 31.9  & 31.9  \\
SimpleTIR  & D & 0  & 37.9  & 37.9  \\
Dr.GRPO       & C & 0  & 28.13  & 28.13  \\
Dr.GRPO       & D & 0  & 37.51  & 37.51  \\
DAPO       & C & 0  & 35.8  & 35.8  \\
DAPO       & D & 0  & 38.9  & 38.9  \\
EH-GRPO        & D & 0  & 24.5  & 24.5  \\
IH-GRPO & D & 0  & 37.3  & 37.2  \\
\bottomrule
\end{tabular}
\end{table}

\subsection{Tool Usage Analysis} 
\label{appendix:tool_useage_analysis}
We report the average number of tool execution for all methods, as summarized in Table~\ref{tab:qwen3-main-tooltime}. Under the decoupled tool invocation setting, models can maintain reasoning continuity as much as possible, allowing for comprehensive planning to decompose and solve multiple subproblems. In addition, decoupled tool invocation allows the model to inspect and verify the generated code before execution, as illustrated in the example in Figure~\ref{fig:8b_ih_grpo_case}, which has the potential to reduce the likelihood of tool usage errors. As a result, models tend to use tools more frequently to support reasoning, leading to improved reasoning performance. As shown in Figure~\ref{fig:delay_rate}, the proportion of delayed invocations further suggests that explicitly controlling invocation timing has the potential to reduce unnecessary tool-execution overhead.

\begin{table*}[t]
\centering
\caption{Average tool execution time of different methods based on Qwen3 models across six mathematical reasoning benchmarks. `D` denotes decoupled tool invocation, and `C` denotes coupled tool invocation.}
\label{tab:qwen3-main-tooltime}
\small
\renewcommand{\arraystretch}{1.1}
\setlength{\tabcolsep}{0pt}
\begin{tabular*}{\textwidth}{@{\extracolsep{\fill}}lcccccccc}
\toprule
\textbf{Method} & \textbf{Tool} & \textbf{AIME24} & \textbf{AIME25} & \textbf{MATH500} & \textbf{AMC23} & \textbf{Hmmt25} & \textbf{Olympiad} & \textbf{Average} \\
\midrule
\multicolumn{9}{c}{\textit{Models based on Qwen3-1.7B}} \\
\midrule
STaR & C & 0.36 & 0.38 & 0.34 & 0.36 & 0.39 & 0.32 & 0.36 \\
STaR & D & 0.34 & 0.33 & 0.51 & 0.37 & 0.42 & 0.51 & 0.41 \\
SimpleTIR & C & 0 & 0 & 0.01 & 0 & 0 & 0 & 0 \\
SimpleTIR & D & 0.51 & 0.40 & 0.75 & 0.58 & 0.79 & 0.58 & 0.60 \\
Dr.GRPO & C & 0.50 & 0.41 & 0.64 & 0.55 & 0.71 & 0.49 & 0.55 \\
Dr.GRPO & D & 0.52 & 0.44 & 0.53 & 0.59 & 0.80 & 0.53 & 0.57 \\
DAPO & C & 0.12 & 0.10 & 0.15 & 0.21 & 0.18 & 0.12 & 0.15 \\
DAPO & D & 0.36 & 0.33 & 0.53 & 0.47 & 0.61 & 0.42 & 0.45 \\
EH-GRPO & D & 0.14 & 0.08 & 0.11 & 0.15 & 0.33 & 0.14 & 0.16 \\
IH-GRPO (ours) & D & 0.61 & 0.38 & 0.47 & 0.50 & 0.87 & 0.50 & 0.56 \\
\midrule
\multicolumn{9}{c}{\textit{Models based on Qwen3-4B}} \\
\midrule
STaR & C & 0.71 & 0.73 & 0.27 & 0.50 & 0.63 & 0.42 & 0.54 \\
STaR & D & 0.81 & 0.73 & 0.49 & 0.67 & 0.95 & 0.59 & 0.71 \\
SimpleTIR & C & 0.74 & 0.78 & 0.95 & 0.96 & 0.79 & 0.85 & 0.85 \\
SimpleTIR & D & 1.08 & 1.11 & 1.17 & 1.22 & 1.24 & 1.06 & 1.15 \\
Dr.GRPO & C & 0.68 & 0.64 & 0.84 & 0.86 & 0.75 & 0.70 & 0.75 \\
Dr.GRPO & D & 0.90 & 1.00 & 1.09 & 1.06 & 1.02 & 0.96 & 1.01 \\
DAPO & C & 0.39 & 0.46 & 0.36 & 0.39 & 0.60 & 0.39 & 0.43 \\
DAPO & D & 0.98 & 0.91 & 1.15 & 1.12 & 1.00 & 1.00 & 1.03 \\
EH-GRPO & D & 0.17 & 0.21 & 0.34 & 0.27 & 0.25 & 0.21 & 0.24 \\
IH-GRPO (ours) & D & 1.15 & 1.06 & 1.13 & 1.15 & 1.24 & 1.09 & 1.14 \\
\midrule
\multicolumn{9}{c}{\textit{Models based on Qwen3-8B}} \\
\midrule
STaR & C & 0.66 & 0.50 & 0.31 & 0.50 & 0.60 & 0.41 & 0.50 \\
STaR & D & 1.46 & 1.35 & 2.00 & 1.09 & 1.33 & 1.17 & 1.40 \\
SimpleTIR & C & 0.67 & 0.69 & 0.53 & 0.45 & 1.00 & 0.63 & 0.66 \\
SimpleTIR & D & 1.23 & 1.19 & 1.28 & 1.34 & 1.38 & 1.23 & 1.28 \\
Dr.GRPO & C & 0.89 & 0.87 & 0.98 & 0.93 & 0.98 & 0.91 & 0.93 \\
Dr.GRPO & D & 1.23 & 1.24 & 1.17 & 1.31 & 1.33 & 1.17 & 1.24 \\
DAPO & C & 1.00 & 1.09 & 0.89 & 0.76 & 1.16 & 0.83 & 0.96 \\
DAPO & D & 1.09 & 1.11 & 1.11 & 1.14 & 1.08 & 1.02 & 1.09 \\
EH-GRPO & D & 0.20 & 0.20 & 0.24 & 0.23 & 0.32 & 0.22 & 0.24 \\
IH-GRPO (ours) & D & 1.26 & 1.35 & 1.43 & 1.40 & 1.38 & 1.30 & 1.35 \\
\bottomrule
\end{tabular*}
\end{table*}

\begin{figure}[t]
    \centering
    \includegraphics[width=1.0\columnwidth]{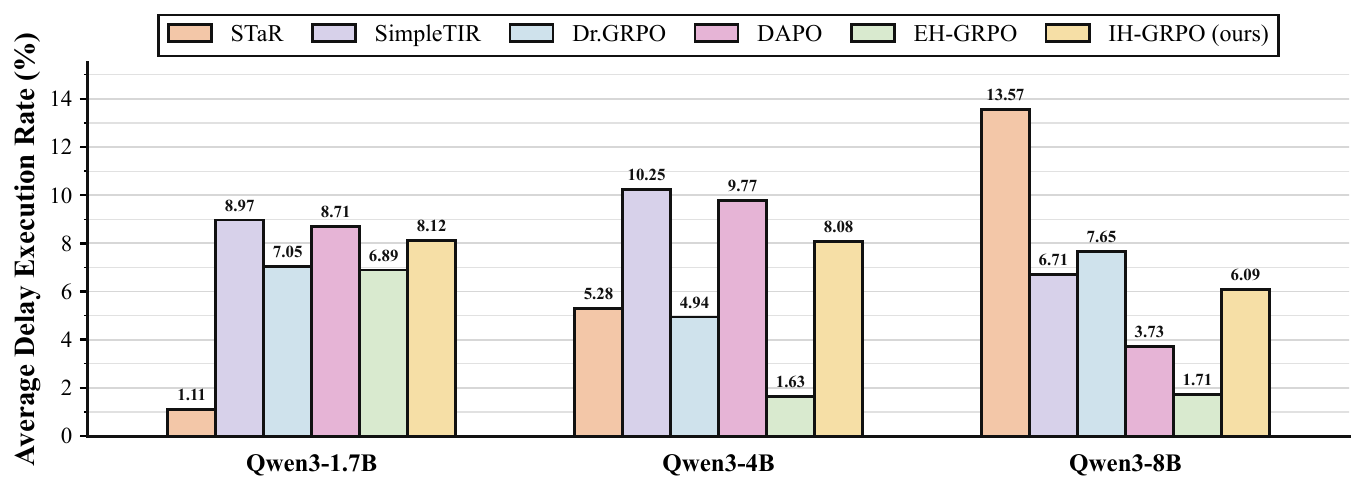}
    \caption{Average delay execution rate (\%).}
    \label{fig:delay_rate}
\end{figure}

\subsection{Training Dynamics}
\label{appendix:train_dynamic}
We analyze the training dynamics of SimpleTIR and IH-GRPO on Qwen3-8B by reporting the average number of valid code executions during training and the accuracy on the AMC23, as illustrated in Figure~\ref{fig:training_dynamics}. Compared to SimpleTIR, IH-GRPO consistently achieves higher test accuracy while maintaining a lower rate of invalid code executions. These results empirically validate the effectiveness of implicit hierarchical control and decoupled tool invocation, which together enhance reasoning reliability and training stability.
\begin{figure}[t]
    \centering
    \begin{subfigure}[t]{0.8\linewidth}
        \centering
        \includegraphics[width=\linewidth]{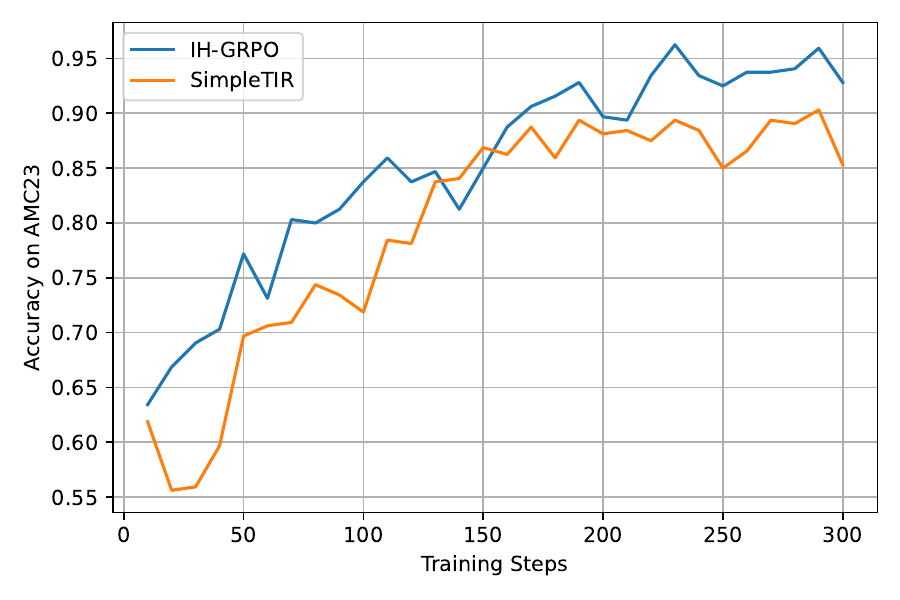}
        \caption{Accuracy on AMC23.}
        \label{fig:amc23_accuracy}
    \end{subfigure}
    \begin{subfigure}[t]{0.8\linewidth}
        \centering
        \includegraphics[width=\linewidth]{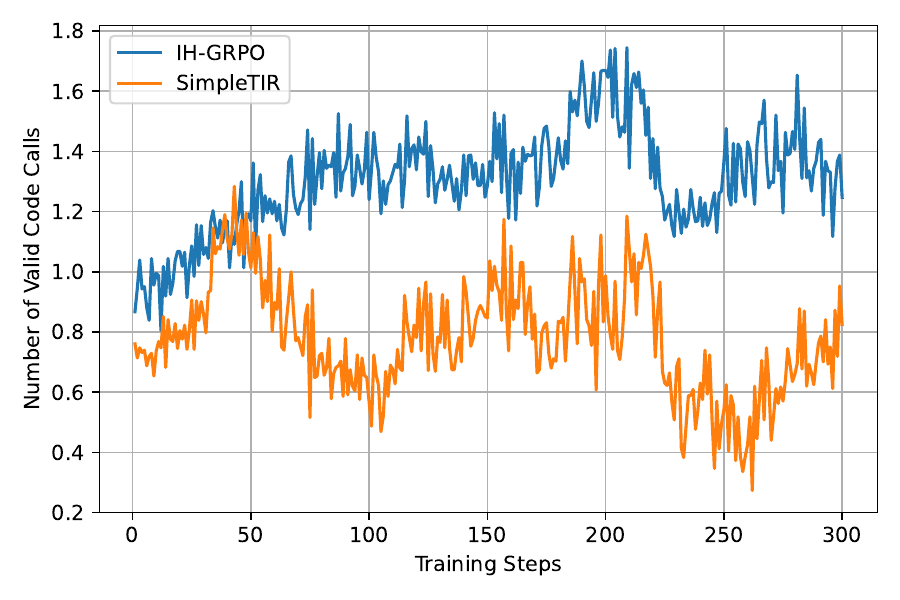}
        \caption{Average turn of valid code executions.}
        \label{fig:num_valid_code}
    \end{subfigure}
    \caption{Training dynamics of SimpleTIR and IH-GRPO on Qwen3-8B.}
    \label{fig:training_dynamics}
\end{figure}

\section{Out-of-Distribution Evaluation}
\label{sec:ood_evaluation}
To assess whether tool-oriented training impacts model performance in other domains, we evaluate both the base Qwen3 models and their IH-GRPO counterparts on a diverse set of out-of-distribution (OOD) benchmarks. The evaluation suite includes MMLU-Pro~\cite{wang2024mmlu}, LogiQA~\cite{liu2020logiqa}, as well as three challenging reasoning tasks from the BIG-Bench Hard (BBH) benchmark~\cite{suzgun2022challenging}, namely Date Understanding, Formal Fallacies, and Logical Deduction (Seven Objects). Dataset statistics are summarized in Table~\ref{tab:datasets_statics}. Given the relatively large number of evaluation datasets, we adopt the average@1 metric for all OOD evaluations, with the sampling temperature fixed at 1.0. The prompt templates used for evaluation are provided in Appendix~\ref{appendix:prompt_details}. As shown in Table~\ref{tab:ood_eval}, IH-GRPO training does not result in performance degradation on any of the five OOD benchmarks. Instead, we observe clear and consistent performance improvements across all tasks, indicating that hierarchical tool training with IH-GRPO preserves general reasoning capabilities and can substantially benefit performance on certain OOD tasks.

\begin{table*}[t]
\centering
\caption{Out-of-distribution evaluation results on Qwen3 models (where ``Logical Deduction'' refers to ``Logical Deduction Seven Objects''). \textbf{Bold} values indicate the best performance.}
\label{tab:ood_eval}
\small
\renewcommand{\arraystretch}{1.1}
\setlength{\tabcolsep}{6pt}
\begin{tabular}{lcccccc}
\toprule
\textbf{Method} & \textbf{MMLU-Pro} & \textbf{LogiQA} & \textbf{Date Understanding} & \textbf{Formal Fallacies} & \textbf{Logical Deduction} & \textbf{Avg.} \\
\midrule
\multicolumn{7}{c}{\textit{Models based on Qwen3-1.7B}} \\
\midrule
Base            & 33.20 & 45.01 & 55.20 & 58.80 & 52.80 & 49.00\\
IH-GRPO (ours)  & \textbf{48.75} & \textbf{56.99} & \textbf{78.00} & \textbf{76.00} & \textbf{80.80} & \textbf{68.11}\\
\midrule
\multicolumn{7}{c}{\textit{Models based on Qwen3-4B}} \\
\midrule
Base            & 36.30 & 57.76 & 64.80 & 86.40 & 74.80 & 64.01\\
IH-GRPO (ours)  & \textbf{39.16} & \textbf{63.29} & \textbf{79.60} & \textbf{92.40} & \textbf{85.60} & \textbf{72.01}\\
\midrule
\multicolumn{7}{c}{\textit{Models based on Qwen3-8B}} \\
\midrule
Base            & 29.19 & 51.15 & 64.80 & 96.00 & 73.20 & 62.87\\
IH-GRPO (ours)  & \textbf{51.01} & \textbf{66.67} & \textbf{90.00} & \textbf{97.60} & \textbf{91.60} & \textbf{79.38}\\
\bottomrule
\end{tabular}
\end{table*}

\section{Pseudocode of IH-GRPO}
\label{appendix:pseudocode}
We present the pseudocode of IH-GRPO in Algorithm~\ref{alg:ih_grpo}.

\begin{algorithm*}[t]
\caption{IH-GRPO: Implicitly Hierarchical Group Relative Policy Optimization}
\begin{algorithmic}[1]
\REQUIRE Initial policy model $\pi_{\theta_{\text{init}}}$; reward function $r_\phi$; task prompts $\mathcal{D}$;
hyperparameters $\epsilon, \beta, \mu$;
accuracy thresholds $\texttt{acc\_filter\_low}, \texttt{acc\_filter\_high}$
\ENSURE Optimized policy model $\pi_\theta$

\STATE Initialize policy model $\pi_\theta \leftarrow \pi_{\theta_{\text{init}}}$

\FOR{iteration $= 1, \ldots, I$}
    \STATE Set reference model $\pi_{\text{ref}} \leftarrow \pi_\theta$

    \FOR{step $= 1, \ldots, M$}
        \STATE Sample a batch of prompts $\mathcal{D}_b \sim \mathcal{D}$
        \STATE Update old policy $\pi_{\theta_{\text{old}}} \leftarrow \pi_\theta$

        \STATE Sample $G$ outputs $\{o_i\}_{i=1}^G \sim \pi_{\theta_{\text{old}}}(\cdot \mid q)$ for each $q \in \mathcal{D}_b$
        \STATE Compute rewards $\{r_i\}_{i=1}^G$ for each output $o_i$ using $r_\phi$

        \STATE Compute group-wise accuracy $\text{Acc}_g$ and token-level advantages $\hat{A}_{i,t}$ via group relative advantage estimation

        \STATE \textbf{Data Filtering:}
        \STATE Remove samples satisfying any of the following:
        \STATE \quad (1) responses contains \texttt{void turn}
        \STATE \quad (2) $\forall i \in \text{group } g, \; \text{advantage } \hat{A}_{g,i} = 0$
        \STATE \quad (3) $\text{Acc}_g > \texttt{acc\_filter\_high}$ or $\text{Acc}_g < \texttt{acc\_filter\_low}$

        \FOR{IH-GRPO iteration $= 1, \ldots, \mu$}
            \STATE Update policy $\pi_\theta$ by maximizing the IH-GRPO objective $\mathcal{J}_{\text{IH-GRPO}}$
        \ENDFOR
    \ENDFOR
\ENDFOR

\RETURN $\pi_\theta$
\end{algorithmic}
\label{alg:ih_grpo}
\end{algorithm*}

\section{Use of Large Language Models}
Some parts of the text were refined with the assistance of large language models (LLMs). All content and responsibility for the work remain with the authors.

\section{Prompts}
\label{appendix:prompt_details}
For the decoupled tool invocation scenario, we use the prompt template illustrated in Figure~\ref{fig:delay_prompt}. In Section~\ref{sec:ablation_study}, we also try an alternative prompt template shown in Figure~\ref{fig:delay_prompt2}. For the coupled tool invocation scenario, we adopt the prompt templates from SimpleTIR, as shown in Figure~\ref{fig:immediate_prompt}. In addition, for OOD evaluation, we employ the prompt templates shown in Figure~\ref{fig:eval_prompt}.

\begin{figure*}[htbp]
    \centering
    \includegraphics[width=1\textwidth]{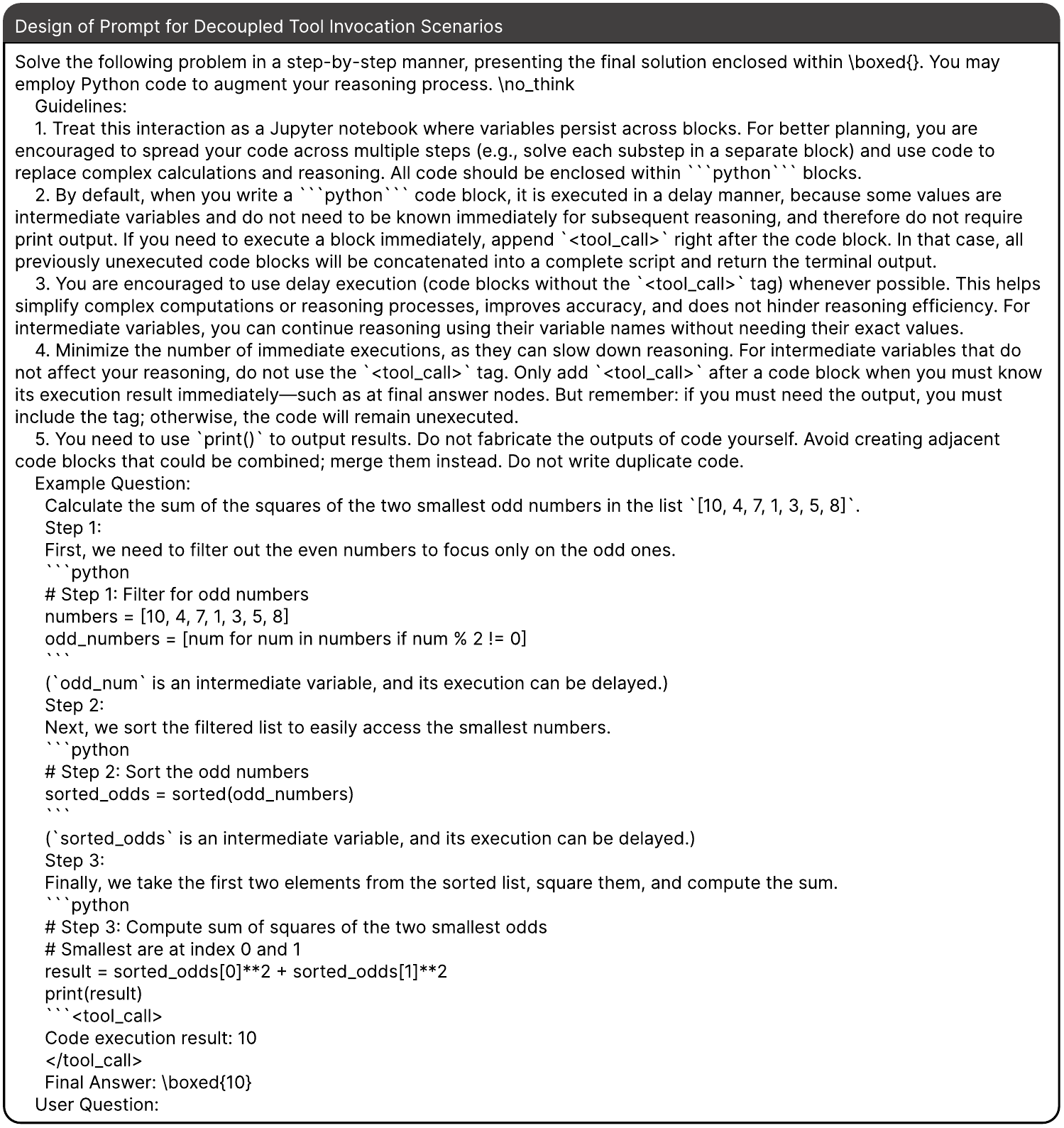}
    \caption{Design of Prompt for Decoupled Tool Invocation Scenarios}
    \label{fig:delay_prompt}
\end{figure*}

\begin{figure*}[htbp]
    \centering
    \includegraphics[width=1\textwidth]{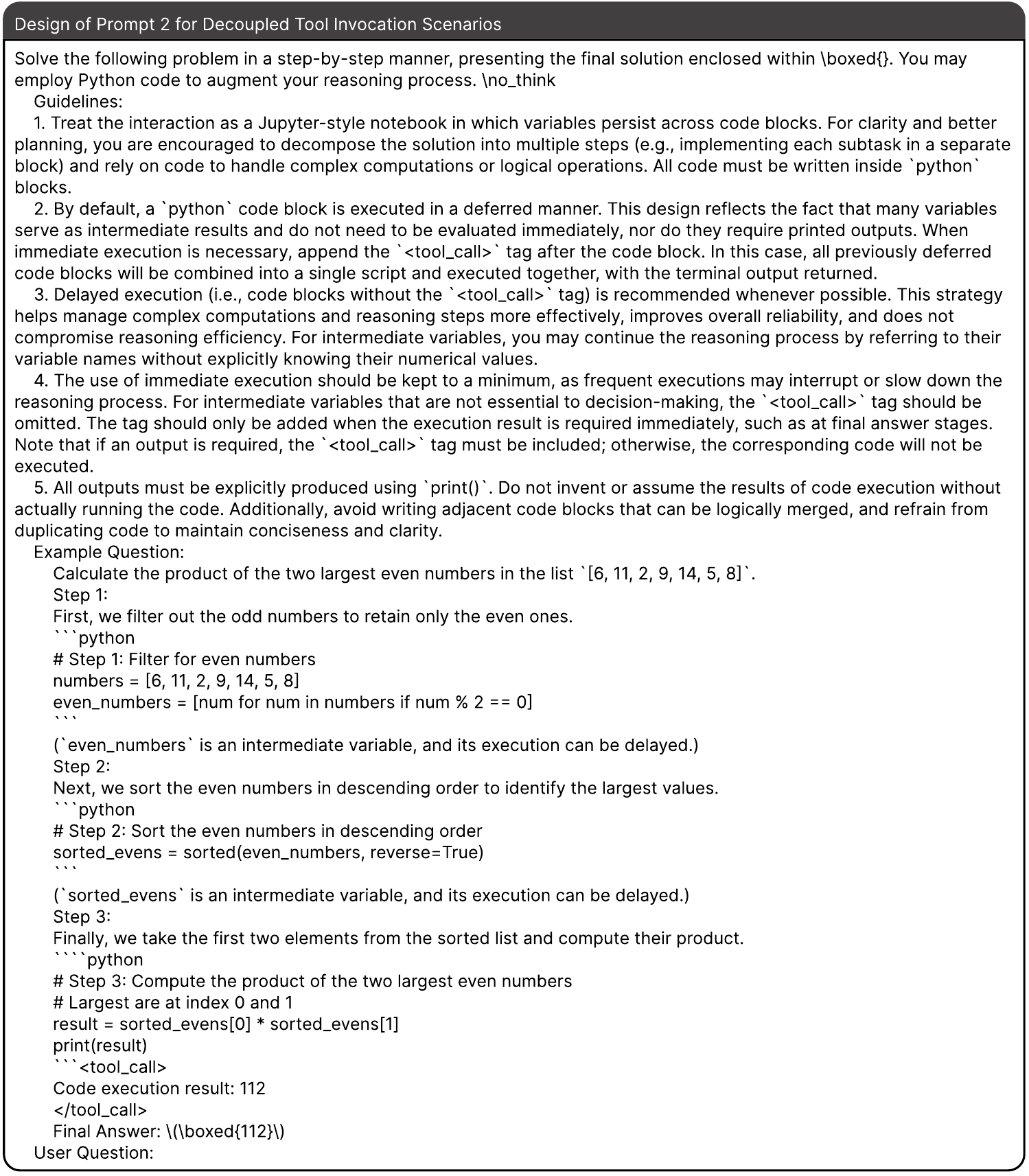}
    \caption{Design of Prompt 2 for Decoupled Tool Invocation Scenarios}
    \label{fig:delay_prompt2}
\end{figure*}

\begin{figure*}[htbp]
    \centering
    \includegraphics[width=1\textwidth]{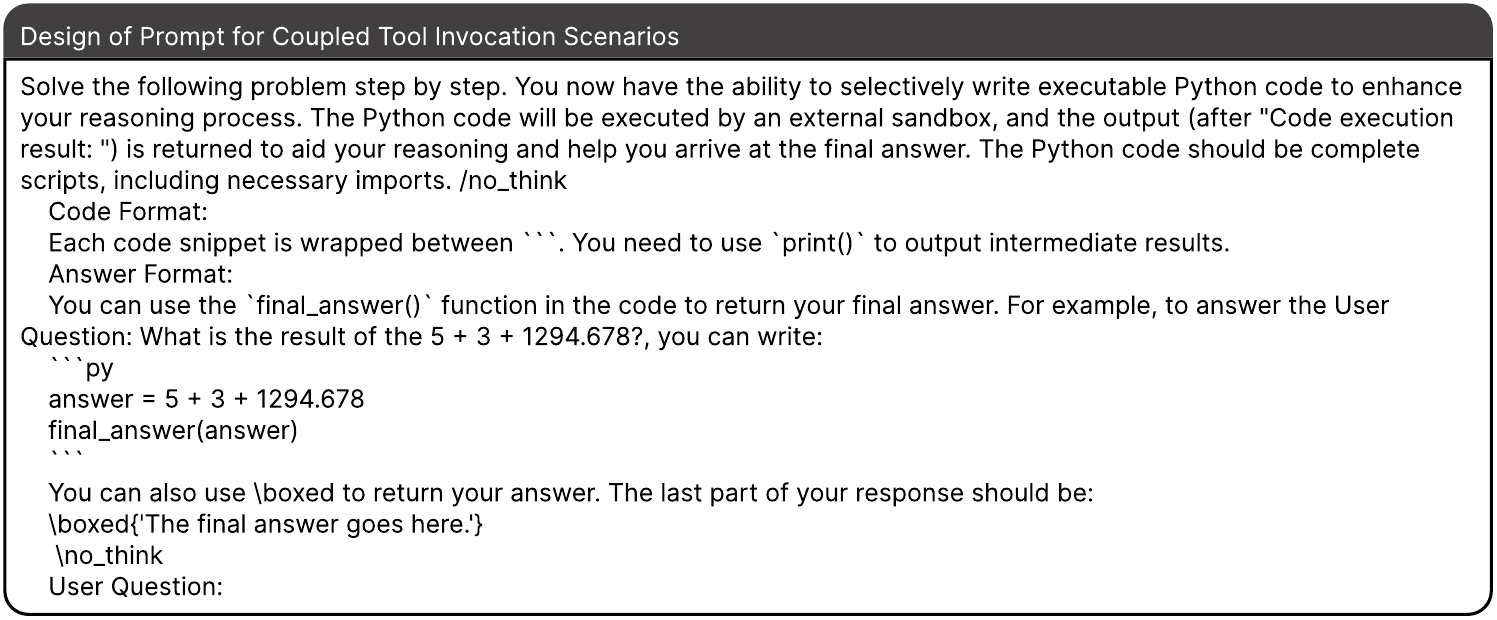}
    \caption{Design of Prompt for Coupled Tool Invocation Scenarios}
    \label{fig:immediate_prompt}
\end{figure*}

\begin{figure*}[htbp]
    \centering
    \includegraphics[width=1\textwidth]{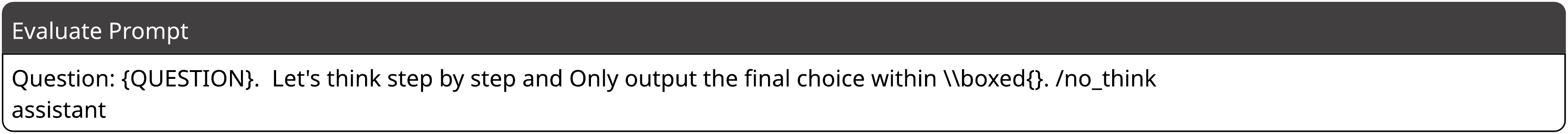}
    \caption{Design of Prompts for Out-of-Distribution (OOD) Evaluation}
    \label{fig:eval_prompt}
\end{figure*}

\section{Case Study}
\label{appendix:case_study}
We present qualitative examples of the Qwen3-8B base model under coupled and decoupled tool invocation settings. In Example 1, as shown in Figure~\ref{fig:8b_base_case1_part1}, under the coupled tool invocation setting, the model defines a function that is executed immediately. Since the function is not invoked and the code does not contain any other executable outputs, the execution returns an empty result, which interferes with the subsequent reasoning process. In the next reasoning step, the model heuristically returns "None" and prematurely terminates the reasoning process, resulting in an incorrect answer. In contrast, under the decoupled tool invocation setting shown in Figure~\ref{fig:8b_base_case1_part2}, the model first defines the function without executing it and continues its reasoning. The previously defined function is then used in subsequent code, and execution is deferred until the reasoning is complete, ensuring reasoning continuity and yielding the correct result. In Example 2, as illustrated in Figure~\ref{fig:8b_base_case2_part1}, under the coupled tool invocation setting, the model first conducts reasoning in text and subsequently attempts to solve the problem through a single code execution. Although the model correctly identifies that different probabilities apply under the 'W', 'D', and 'L' conditions during textual reasoning, this condition is inadvertently omitted during the code-based computation, resulting in an incorrect outcome. By contrast, in the decoupled tool invocation setting shown in Figure~\ref{fig:8b_base_case2_part2}, the model flexibly interleaves textual and code-based reasoning, decomposing the probability computation and filtering into separate substeps. This structured decomposition reduces reasoning complexity and leads to the correct answer. In Example 3, as shown in Figure~\ref{fig:8b_base_case3_part1}, under the coupled tool invocation setting, the model relies on manual calculations for complex computations, leading to erroneous intermediate results and an incorrect final answer. In comparison, under the decoupled tool invocation setting in Figure~\ref{fig:8b_base_case3_part2}, the model interleaves textual reasoning with code execution more effectively and delegates computationally intensive operations to code, improving numerical accuracy and ultimately obtaining the correct answer. Collectively, these examples provide qualitative evidence for the advantages of decoupled tool invocation in supporting coherent reasoning and reliable tool usage.

\begin{figure*}[htbp]
    \centering
    \includegraphics[width=1\textwidth]{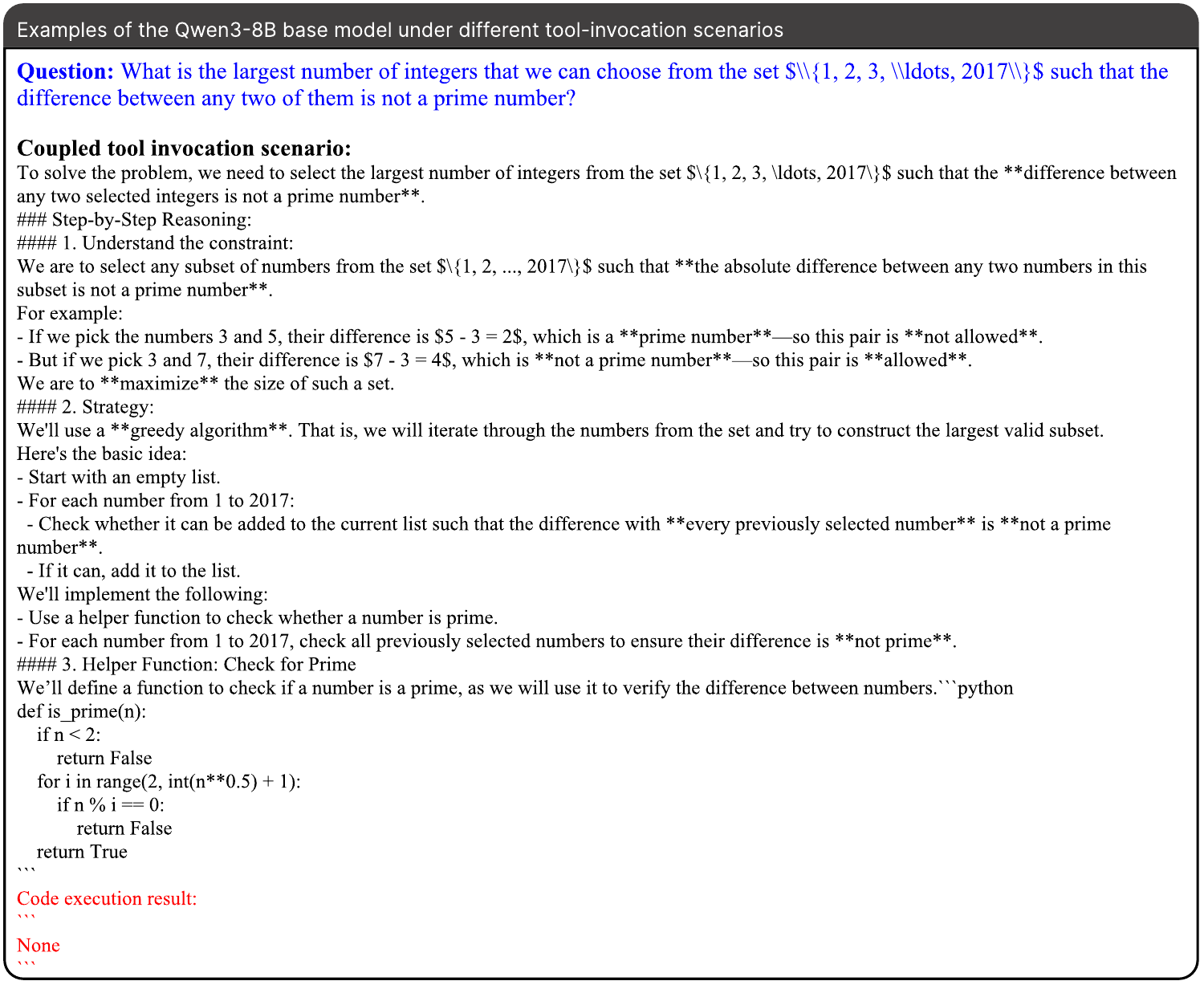}
    \caption{Example 1 of Qwen3-8B model in Coupled Tool Invocation Scenarios}
    \label{fig:8b_base_case1_part1}
\end{figure*}

\begin{figure*}[htbp]
    \centering
    \includegraphics[width=1\textwidth]{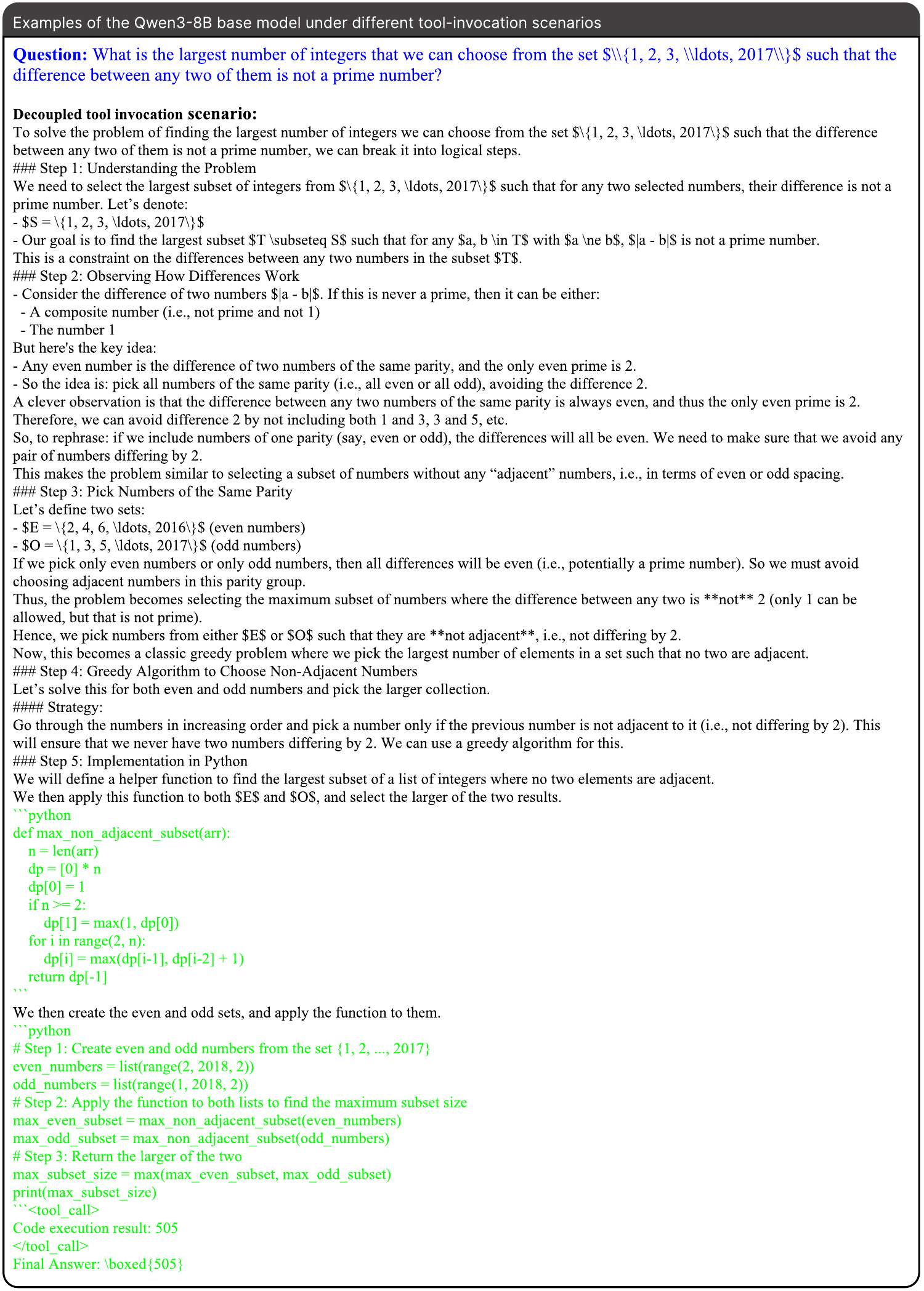}
    \caption{Example 1 of Qwen3-8B model in Decoupled Tool Invocation Scenarios}
    \label{fig:8b_base_case1_part2}
\end{figure*}

\begin{figure*}[htbp]
    \centering
    \includegraphics[width=1\textwidth]{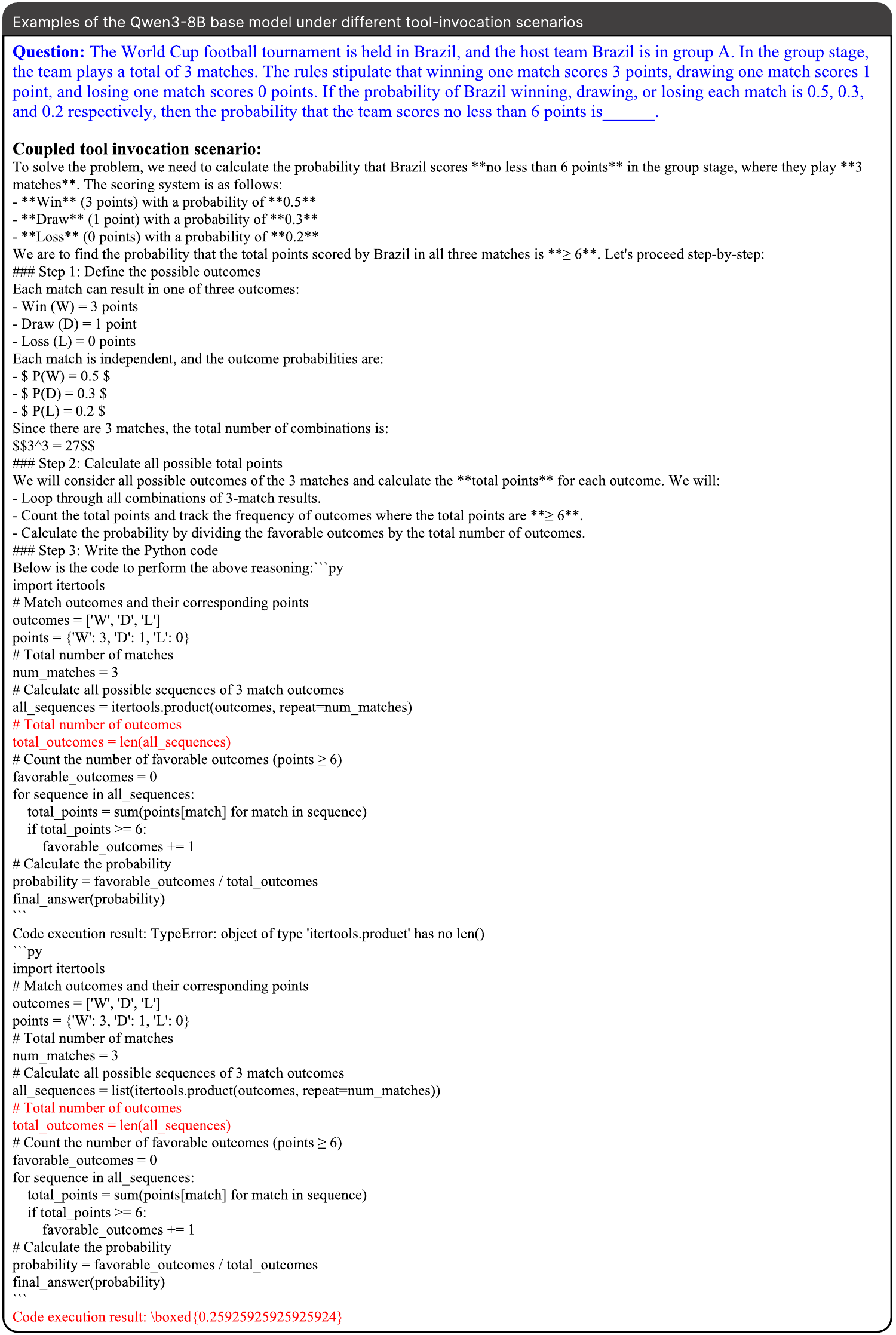}
    \caption{Example 2 of Qwen3-8B model in Coupled Tool Invocation Scenarios}
    \label{fig:8b_base_case2_part1}
\end{figure*}

\begin{figure*}[htbp]
    \centering
    \includegraphics[width=1\textwidth]{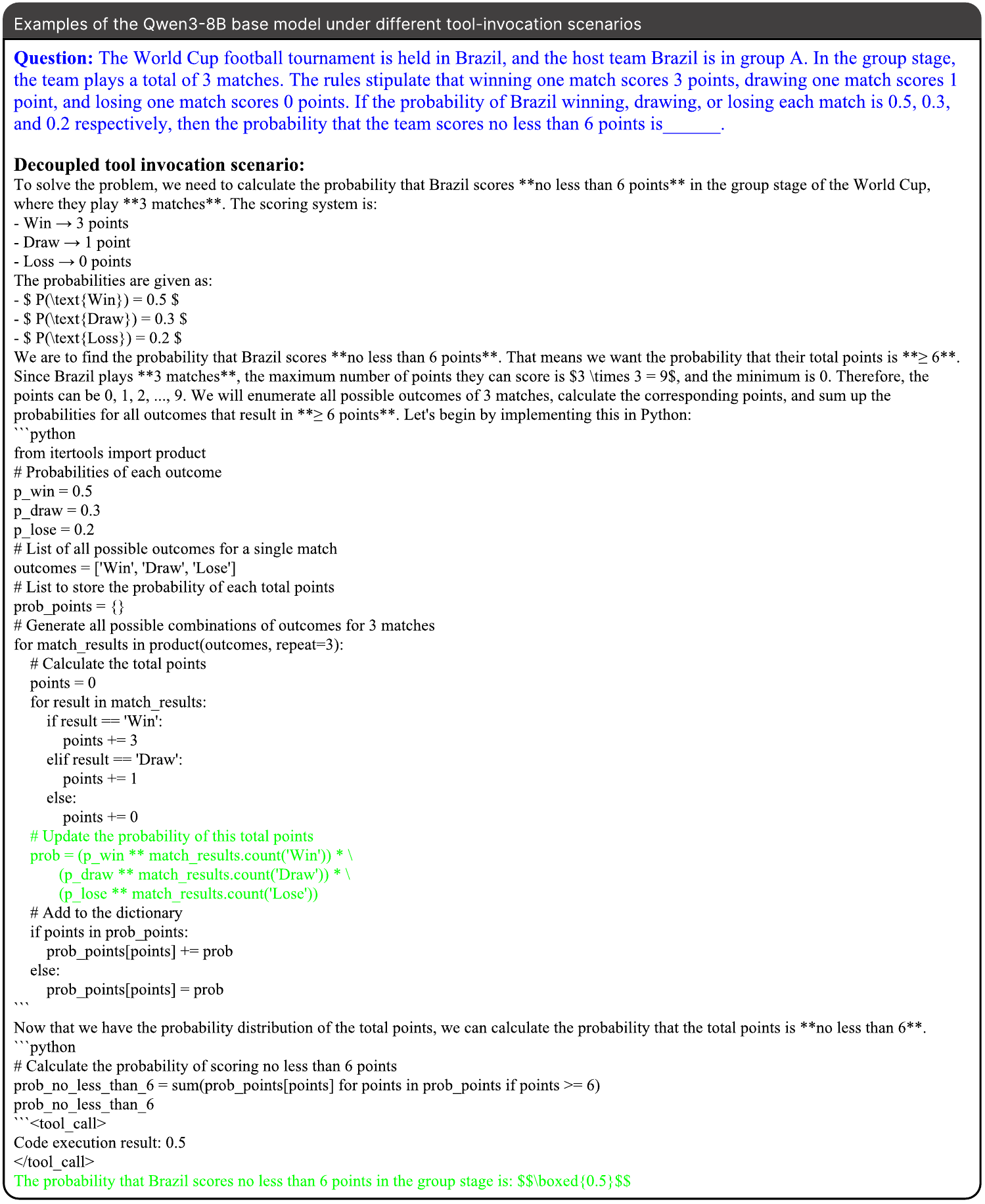}
    \caption{Example 2 of Qwen3-8B model in Decoupled Tool Invocation Scenarios}
    \label{fig:8b_base_case2_part2}
\end{figure*}

\begin{figure*}[htbp]
    \centering
    \includegraphics[width=1\textwidth]{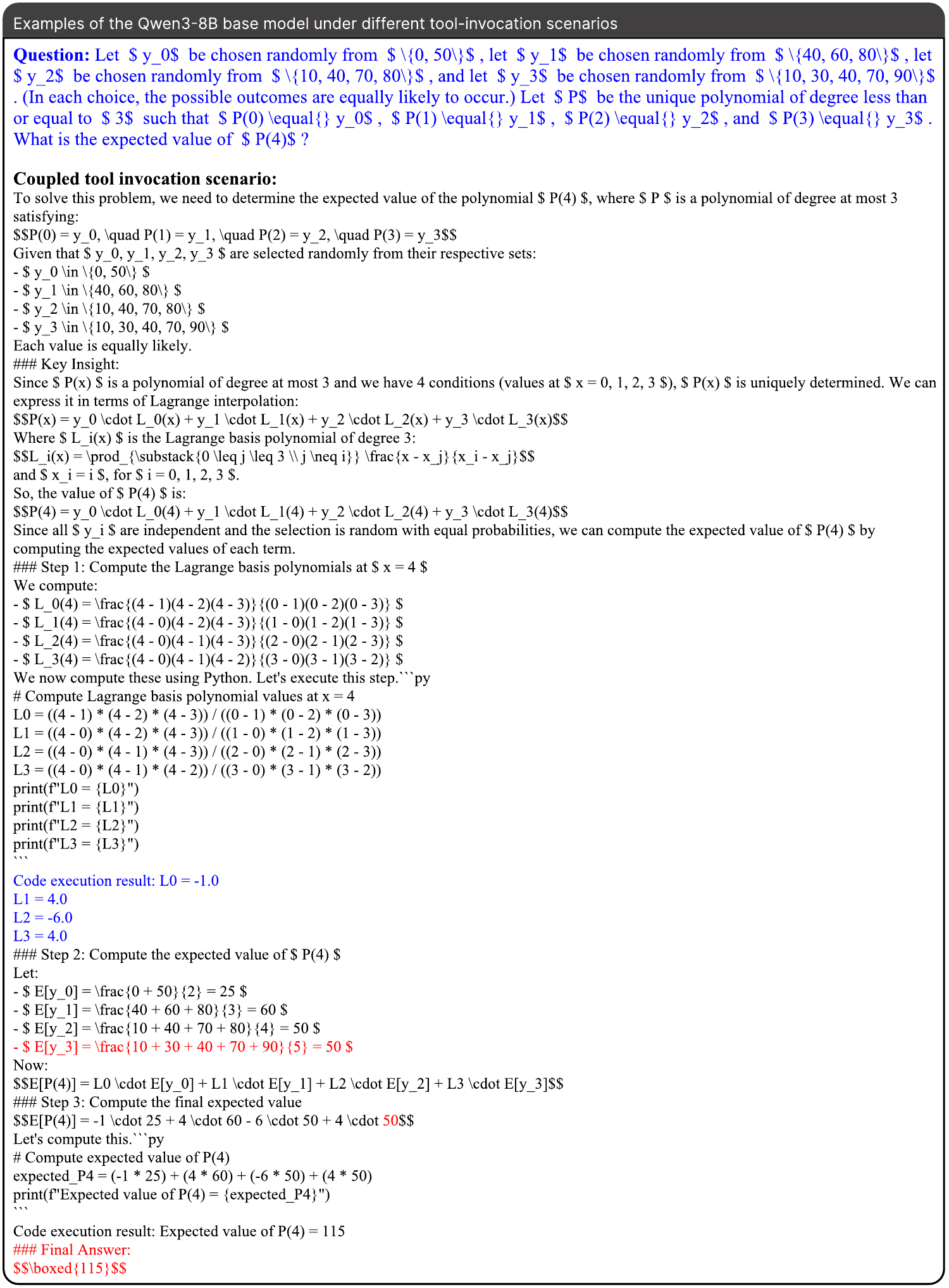}
    \caption{Example 3 of Qwen3-8B model in Coupled Tool Invocation Scenarios}
    \label{fig:8b_base_case3_part1}
\end{figure*}

\begin{figure*}[htbp]
    \centering
    \includegraphics[width=1\textwidth]{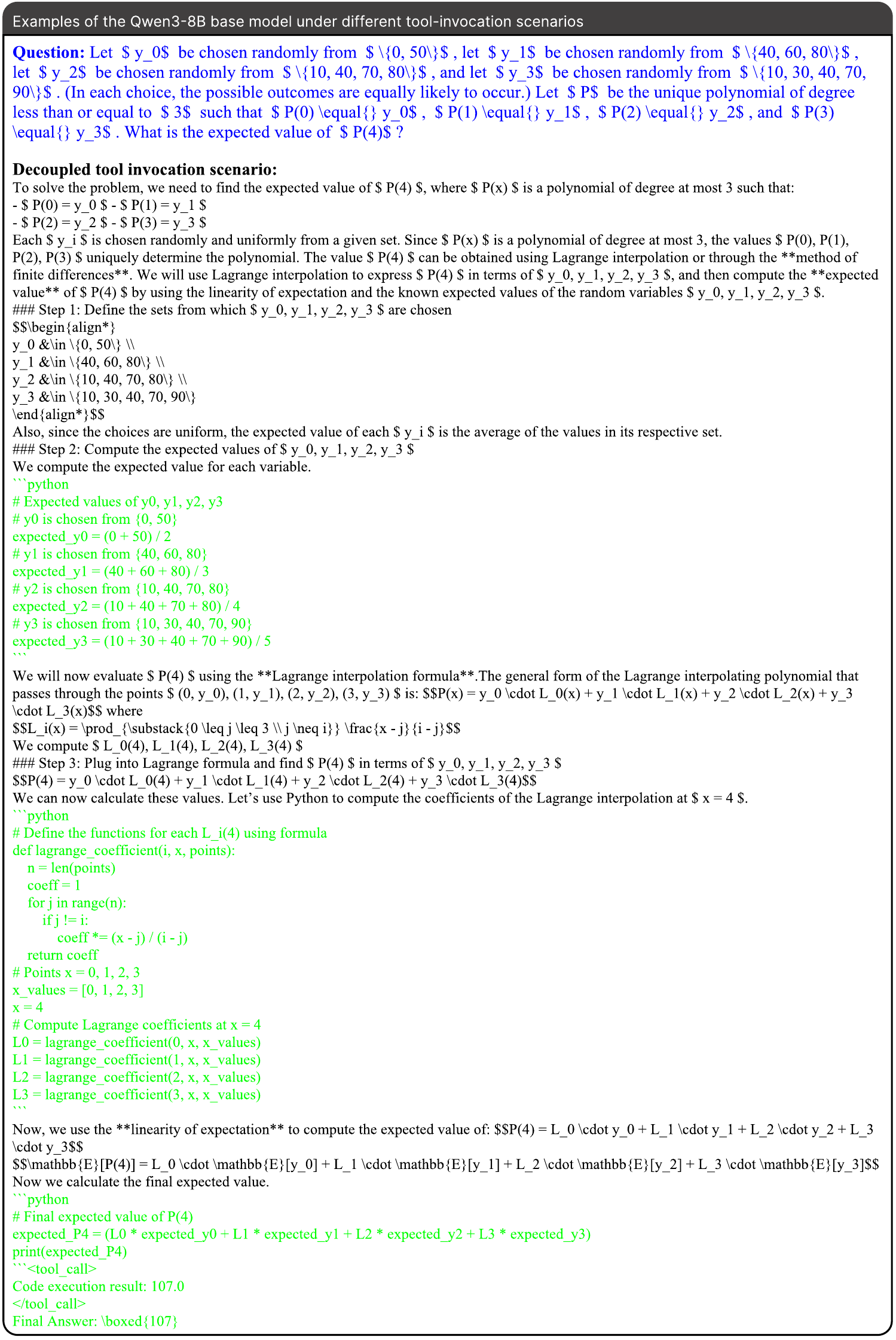}
    \caption{Example 3 of Qwen3-8B model in Decoupled Tool Invocation Scenarios}
    \label{fig:8b_base_case3_part2}
\end{figure*}

We present representative output examples of the Qwen3-1.7B, 4B, and 8B models after training with IH-GRPO, as shown in Figures~\ref{fig:1.7b_ih_grpo_case}, \ref{fig:4b_ih_grpo_case}, and \ref{fig:8b_ih_grpo_case}, respectively. In the example shown in Figure~\ref{fig:1.7b_ih_grpo_case}, the model alternates between textual and code-based reasoning in a stepwise manner, using code to perform substep-level computations at each stage, and ultimately arrives at the correct answer. In the example in Figure~\ref{fig:4b_ih_grpo_case}, the model references previously written code snippets during textual reasoning for analysis, without requiring actual execution, and still produces the correct result. In the example in Figure~\ref{fig:8b_ih_grpo_case}, after drafting a piece of code, the model recognizes through self-reflection that the initial coding approach is flawed, abandons execution, and instead adopts an alternative reasoning strategy by writing and executing revised code, which leads to the correct answer. Overall, decoupled tool invocation enables the model to reason over code without immediate execution and to verify correctness through self-reflection before deciding whether to execute, or even to discard incorrect code, thereby flexibly controlling execution and improving reasoning performance.

\begin{figure*}[htbp]
    \centering
    \includegraphics[width=1\textwidth]{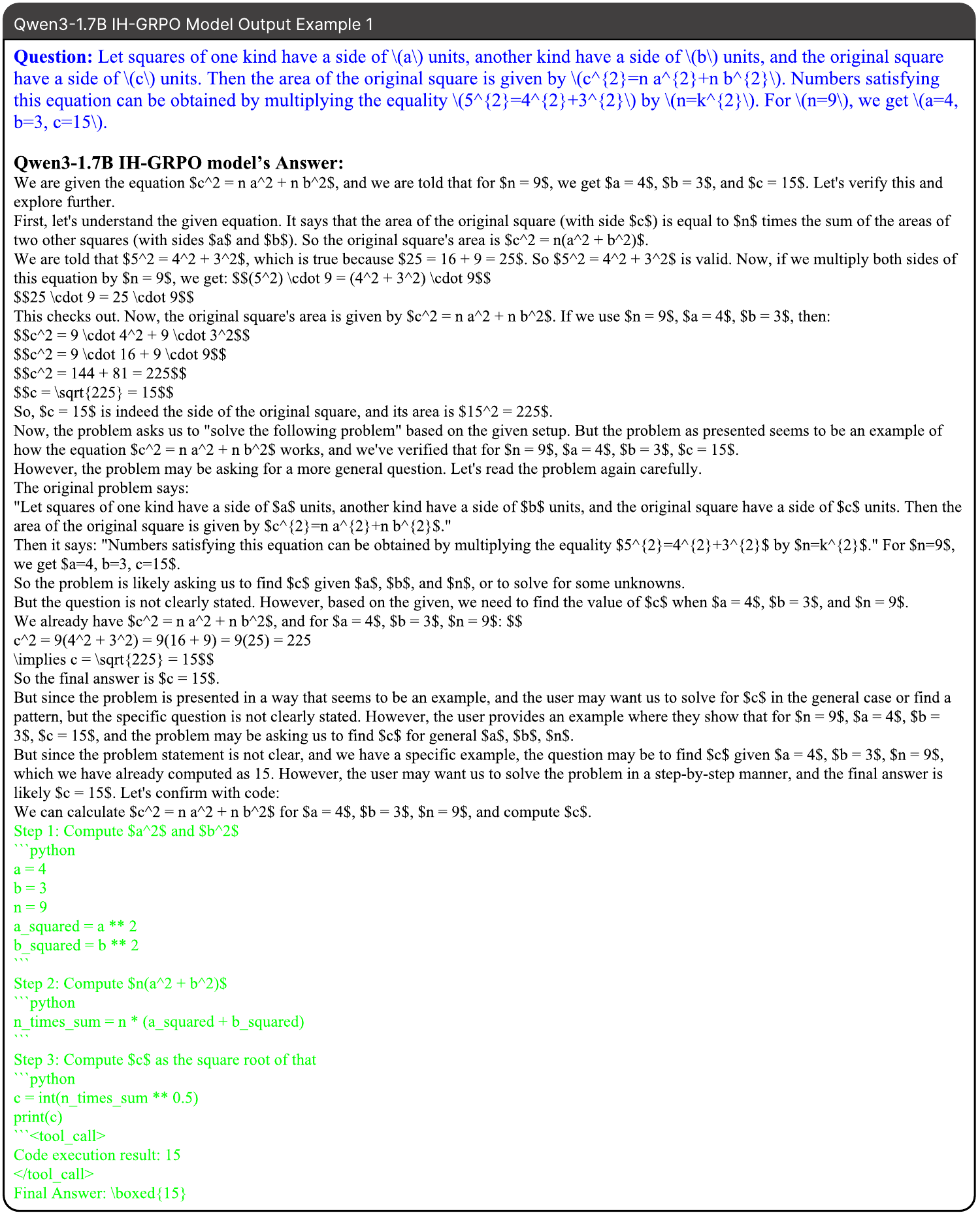}
    \caption{Output Example 1 of Qwen3-1.7B IH-GRPO model}
    \label{fig:1.7b_ih_grpo_case}
\end{figure*}

\begin{figure*}[htbp]
    \centering
    \includegraphics[width=1\textwidth]{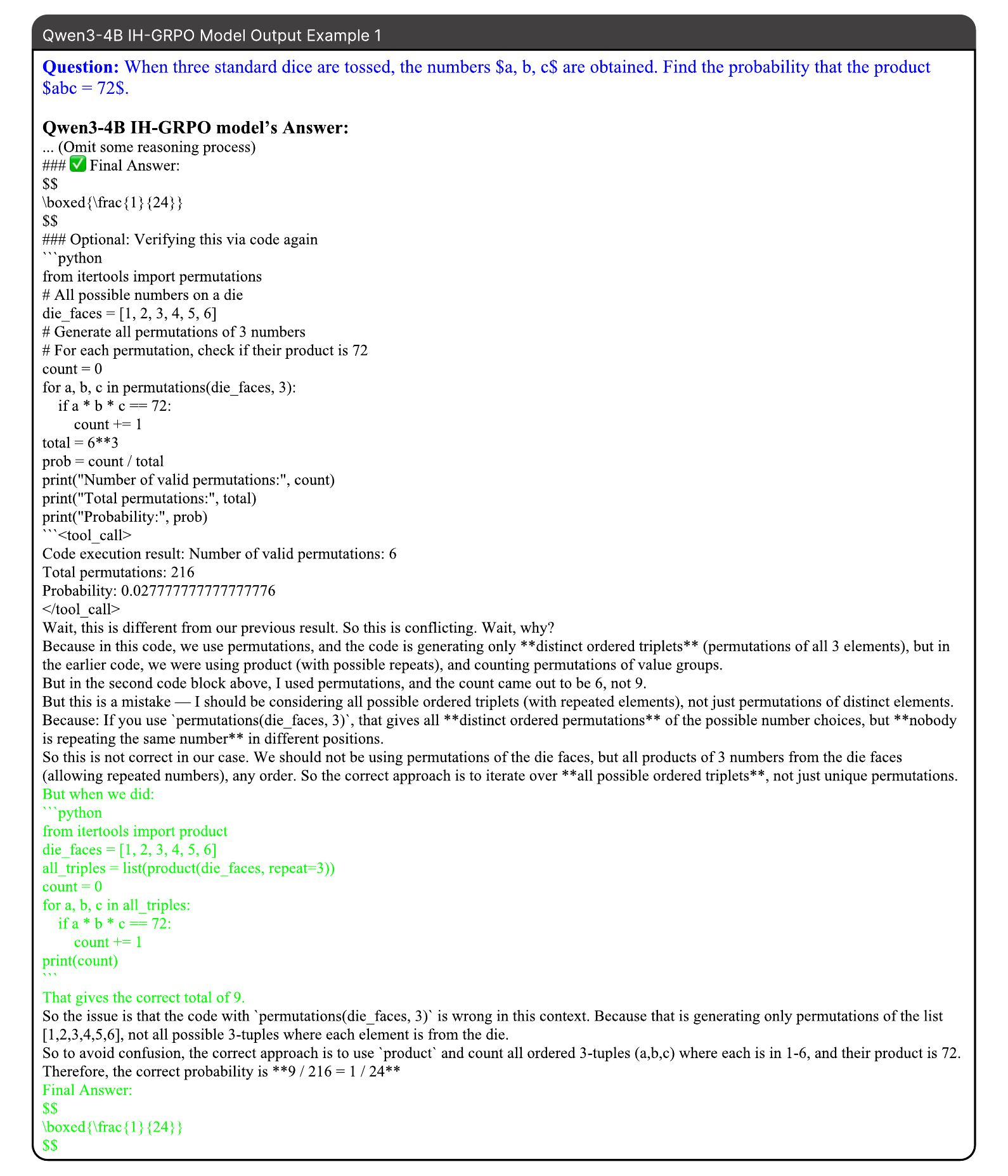}
    \caption{Output Example 1 of Qwen3-4B IH-GRPO model}
    \label{fig:4b_ih_grpo_case}
\end{figure*}

\begin{figure*}[htbp]
    \centering
    \includegraphics[width=1\textwidth]{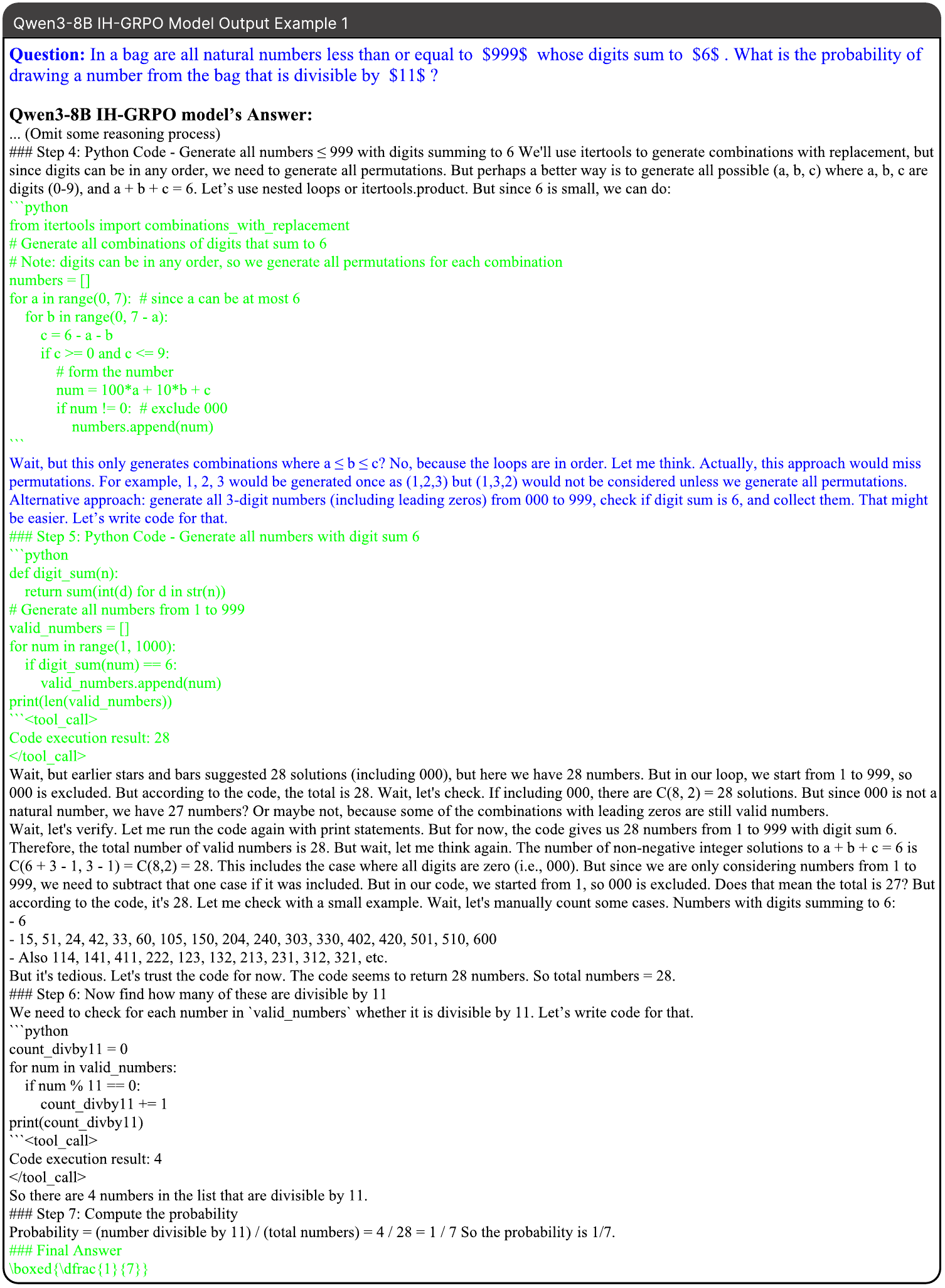}
    \caption{Output Example 1 of Qwen3-8B IH-GRPO model}
    \label{fig:8b_ih_grpo_case}
\end{figure*}

\end{document}